\documentclass{article}
\usepackage{Definitions}
\usepackage{times}
\usepackage{graphicx} 
\usepackage{subfigure}
\usepackage{caption}
\usepackage{natbib}
\usepackage{verbatim}
\usepackage{algorithm}
\usepackage{algorithmic}
\usepackage{hyperref}
\usepackage[dvipsnames]{xcolor}

\usepackage{Definitions}
\usepackage{color}

\renewcommand{\algorithmicrequire}{\textbf{Input:}}
\renewcommand{\algorithmicensure}{\textbf{Output:}}

\usepackage{enumitem}
\usepackage[accepted]{icml2017}

\icmltitlerunning{Iterative Machine Teaching}

\begin{document}

\twocolumn[
\icmltitle{Iterative Machine Teaching}

\begin{icmlauthorlist}
\icmlauthor{Weiyang Liu}{gt}
\icmlauthor{Bo Dai}{gt}
\icmlauthor{Ahmad Humayun}{gt}
\icmlauthor{Charlene Tay}{iu}
\icmlauthor{Chen Yu}{iu}\\
\icmlauthor{Linda B. Smith}{iu}
\icmlauthor{James M. Rehg}{gt}
\icmlauthor{Le Song}{gt}
\end{icmlauthorlist}

\icmlaffiliation{gt}{Georgia Institute of Technology}
\icmlaffiliation{iu}{Indiana University}

\icmlcorrespondingauthor{Weiyang Liu}{wyliu@gatech.edu}
\icmlcorrespondingauthor{Le Song}{lsong@cc.gatech.edu}

\icmlkeywords{Iterative Machine Teaching}

\vskip 0.3in
]

\printAffiliationsAndNotice{} 

\begin{abstract}
In this paper, we consider the problem of machine teaching, the inverse problem of machine learning. Different from traditional machine teaching which views the learners as batch algorithms, we study a new paradigm where the learner uses an iterative algorithm and a teacher can feed examples sequentially and intelligently based on the current performance of the learner. We show that the teaching complexity in the iterative case is very different from that in the batch case. Instead of constructing a minimal training set for learners, our iterative machine teaching focuses on achieving fast convergence in the learner model. Depending on the level of information the teacher has from the learner model, we design teaching algorithms which can provably reduce the number of teaching examples and achieve faster convergence than learning without teachers. We also validate our theoretical findings with extensive experiments on different data distribution and real image datasets.
\end{abstract}
\vspace{-7mm}
\section{Introduction}
\setlength{\abovedisplayskip}{4pt}
\setlength{\abovedisplayshortskip}{1pt}
\setlength{\belowdisplayskip}{4pt}
\setlength{\belowdisplayshortskip}{1pt}
\setlength{\jot}{3pt}
\setlength{\textfloatsep}{4pt}

Machine teaching is the problem of constructing an optimal (usually minimal) dataset according to a target concept such that a student model can learn the target concept based on this dataset. Recently, there is a surge of interests in machine teaching which has found diverse applications in model compression~\cite{bucila2006model,han2015deep,ba2014deep,romero2014fitnets}, transfer learning \cite{pan2010survey} and cyber-security problems~\cite{alfeld2016data,alfeld2017explicit,mei2015using}. Furthermore, machine teaching is also closely related to other subjects of interests, such as curriculum learning~\cite{bengio2009curriculum} and knowledge distilation~\cite{hinton2015distilling}.

In the traditional machine learning paradigm, a teacher will typically construct a batch set of examples, and provide them to a learning algorithm in one shot; then the learning algorithm will work on this batch dataset trying to learn the target concept. Thus, many research work under this topic try to construct the smallest such dataset, or characterize the size of of such dataset, called the teaching dimension of the student model~\cite{zhu2013machine,zhu2015machine}. There are also many seminal theory work on analyzing the teaching dimension of different models~\cite{shinohara1991teachability,goldman1995complexity,doliwa2014recursive,liu2016teaching}.

\begin{figure}[t]
  \centering
  \footnotesize
  \includegraphics[width=2.5in]{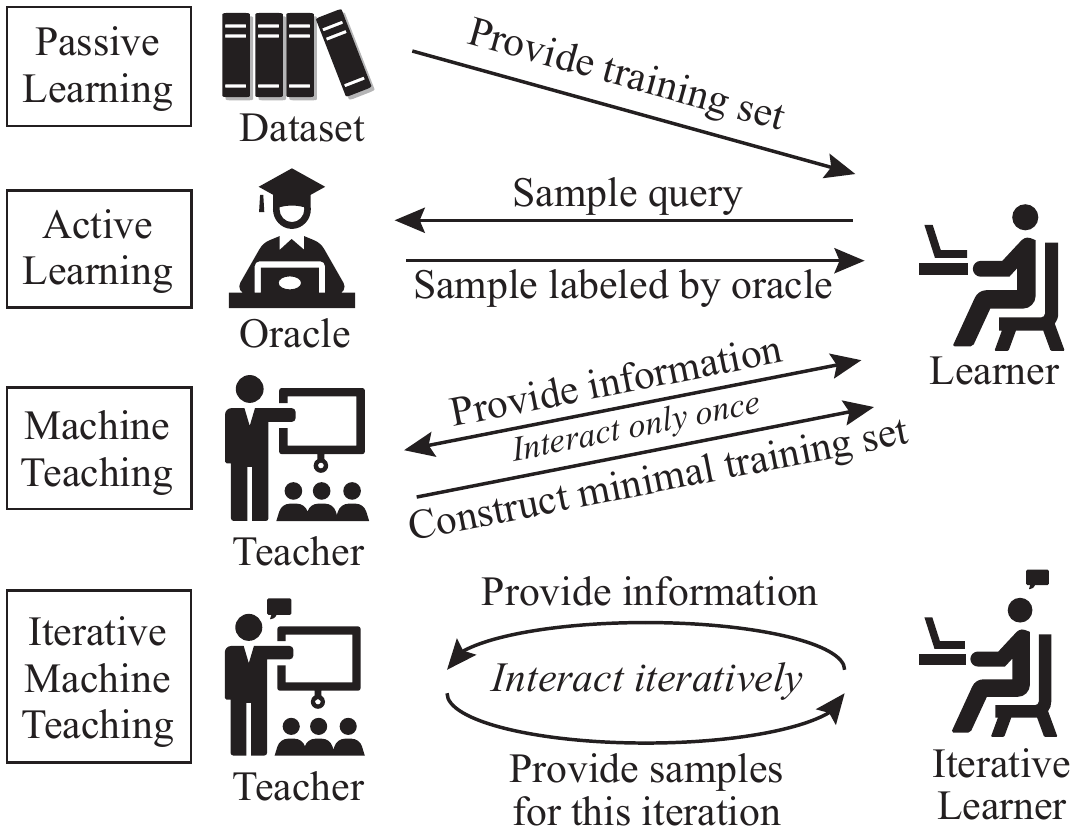}
  \vspace{-2.7mm}
  \caption{Comparison between iterative machine teaching and the other learning paradigms.}\label{fig1}
  \vspace{1.5mm}
\end{figure}
\vspace{-1mm}
However, in many real world applications, the student model is typically updated via an iterative algorithm, and we get the opportunity to observe the performance of the student model as we feed examples to it. For instance,

\vspace{-1.3mm}
\begin{itemize}[leftmargin=*,nosep,nolistsep]
 \item In model compression where we want to transfer a target ``teacher model'' to a destination ``student
  model'', we can constantly observe student model's prediction on current training points. Intuitively, such observations will allow us to get a better estimate where the student model is and pick examples more intelligently to better guide the student model to convergence.
 \item In cyber-security setting where an attack wants to mislead a recommendation system that learns online, the attacker can constantly generate fake clicks and observe the system's response. Intuitively, such feedback will allow the attacker to figure out the state of the learning system, and design better strategy to mislead the system.
\end{itemize}
\vspace{-1.3mm}

From the aspects of both faster model compression and better avoiding hacker attack, we seek to understand some fundamental questions, such as, \emph{what is the sequence of examples that teacher should feed to the student in each iteration in order to achieve fast convergence? And how many such examples or such sequential steps are needed?}

In this paper, we will focus on this new paradigm, called \emph{\textbf{iterative machine teaching}}, which extends traditional machine teaching from batch setting to iterative setting. In this new setting, the teacher model can communicate with and influence the student model in multiple rounds, but the student model remains passive. More specifically, in each round, the teacher model can observe (potentially different levels of) information about the students to intelligently choose one example, and the student model runs a fixed iterative algorithm using this chosen example.
\par
Furthermore, the smallest number of examples (or rounds) the teacher needs to construct in order for the student to efficiently learn a target model is called the \textit{\textbf{iterative teaching dimension}} of the student algorithm. Notice that in this new paradigm, we shift from describing the complexity of a model to the complexity of an algorithm. Therefore, for the same student model, such as logistic regression, the iterative teaching dimension for a teacher model can be different depending on the student's learning algorithms, such as gradient descent versus conjugate gradient descent. In some sense, the teacher in this new setting is becoming active, but not the student. In Fig.~\ref{fig1}, we summarize the differences of iterative machine teaching from traditional machine teaching, active learning and passive learning.
\par
Besides introducing the new paradigm, we also propose three iterative teaching algorithms, called omniscient teacher, surrogate teacher and imitation teacher, based on the level of information about the student that the teacher has access to. Furthermore we provide partial theoretical analysis for these algorithms under different example construction schemes. Our analysis shows that under suitable conditions, iterative teachers can always perform better than passive teacher, and achieve exponential improvements. Our analysis also identifies two crucial properties, namely teaching monotonicity and teacher capability, which play critical roles in achieving fast iterative teaching.

To corroborate our theoretical findings, we also conduct extensive experiments on both synthetic data and real image data. In both cases, the experimental results verify our theoretical findings and the effectiveness of our proposed iterative teaching algorithms.

\vspace{-2.5mm}
\section{Related Work}
\vspace{-0.5mm}
\textbf{Machine teaching.} Machine teaching problem is to find an optimal training set given a student model and a target. \cite{zhu2015machine} proposes a general teaching framework. \cite{zhu2013machine} considers Bayesian learner in exponential family and expresses the machine teaching as an optimization problem over teaching examples that balance the future loss of the learner and the effort of the teacher. \cite{liu2016teaching} provides the teaching dimension of several linear learners. The framework has been applied to security \cite{mei2015using}, human computer interaction \cite{meek2016analysis} and education \cite{khan2011humans}. \cite{JohnsCVPR2015} further extends machine teaching to interactive settings. However, these work ignores the fact that a student model is typically learned by an iterative algorithm, and we usually care more about how fast the student can learn from the teacher.
\par
\vspace{-0.8mm}
\textbf{Interactive Machine Learning.} \cite{cakmak2014eliciting} consider the scenario of a human training an agent to perform a classification task by showing examples. They study how to improve human teacher by giving teaching guidance. \cite{singla2014near} consider the crowdsourcing problem and propose a sequential teaching algorithm that can teach crowd worker to better classify the query. Both work consider a very different setting where the learner (i.e. human learner) is not iterative and does not have a particular optimization algorithm.
\par
\vspace{-0.8mm}
\textbf{Active learning.} Active learning enables a learner to interactively query the oracle to obtain the desired outputs at new samples.
Machine teaching is different from active learning in the sense that active learners explore the optimal parameters by itself rather than being guided by the teacher. Therefore they have different sample complexity \cite{balcan2010true,zhu2013machine}.
\par
\vspace{-0.8mm}
\textbf{Curriculum learning.} Curriculum learning \cite{bengio2009curriculum} is a general training strategy that encourages to input training examples from easy ones to difficult ones. Very interestingly, our iterative teacher model suggests similar training strategy in our experiments.
\vspace{-2.3mm}
\section{Iterative Machine Teaching}
\vspace{-0.7mm}
The proposed iterative machine teaching is a general concept, and the paper considers the following settings:
\par
\vspace{-0.8mm}
\textbf{Student's Asset.} In general, the asset of a student (learner) includes the initial parameter $w_0$, loss function, optimization algorithm, representation (feature), model, learning rate $\eta_t$ over time (and initial $\eta_0$) and the trackability of the parameter $w^t$. The ideal case is that a teacher has access to all of them and can track the parameters and learning rate, while the worst case is that a teacher knows nothing. How practical the teaching is depends on how much the prior knowledge and trackability that a teacher has.
\par
\vspace{-0.8mm}
\textbf{Representation.} The teacher represents an example as $(x, y)$ while the student represents the same example as $(\widetilde x, \widetilde y)$ (typically $\thickmuskip=2mu \medmuskip=2mu y=\widetilde y$). The representation $\thickmuskip=2mu \medmuskip=2mu x\in\mathcal{X}$ and $\thickmuskip=2mu \medmuskip=2mu \widetilde x\in\widetilde\Xcal$ can be different but deterministically related. We assume there exists $\thickmuskip=2mu \medmuskip=2mu \widetilde x = \mathcal{G}(x)$ for an unknown invertible mapping $\mathcal{G}$.
\par
\vspace{-0.8mm}
\textbf{Model.} The teacher uses a linear model $\thickmuskip=2mu \medmuskip=2mu y=\inner{v}{x} $ with parameter $v^*$ ($w^*$ for student's space) that is taught to the student. The student also uses a linear model $\thickmuskip=0mu \medmuskip=0mu \widetilde y = \inner{w}{\widetilde x}$ with parameter $w$, i.e., $\thickmuskip=0mu \medmuskip=0mu \widetilde y = \inner{w}{\Gcal(x)} = f(x)$ in general. $w$ and $v$ do not necessarily lie in the same space, but for omniscient teacher, they are equivalent and interchangeably used.
\par
\textbf{Teaching protocol.} In general, the teacher can only communicate with the student via examples. In this paper, the teacher provides one example $x^t$ in one iteration, where $t$ denotes the $t$-th iteration. The goal of the teacher is to provide examples in each iteration such that the student parameter $w$ converge to its optimum $w^*$ as fast as possible.
\par
\textbf{Loss function.} The teacher and student share the same loss function. We assume this is a convex loss function $\ell(f(x), y)$, and the best model is usually found by minimizing the expected loss below:
\begin{equation}
\footnotesize
\begin{aligned}\label{eq:update}
  w^* = \arg\min_{w}~\EE_{(x,y)} \sbr{\ell(\inner{w}{x}, y)}.
\end{aligned}
\end{equation}
where the sampling distribution $\thickmuskip=2mu \medmuskip=2mu (x,y)\sim \PP(x,y)$. Without loss of generality, we only consider typical loss functions, such as square loss $\thickmuskip=2mu \medmuskip=2mu \frac{1}{2}(\inner{w}{x} - y)^2$, logistic loss $\thickmuskip=2mu \medmuskip=2mu \log(1 + \exp(-y \inner{w}{x}))$ and hinge loss $\thickmuskip=2mu \medmuskip=2mu \max(1-y\inner{w}{x},0)$.
\par
\textbf{Algorithm.} The student uses the stochastic gradient descent to optimize the model. The iterative update is
\begin{equation}
\footnotesize
\begin{aligned}
  w^{t+1} = w^t - \eta_t\, \frac{\partial \ell(\inner{w}{x},y)}{\partial w}.
\end{aligned}
\end{equation}
Without teacher's guiding, the student can be viewed as being guided by a random teacher who randomly feed an example to the student in each iteration.

\vspace{-2.55mm}
\section{Teaching by an Omniscient Teacher}
\vspace{-1mm}
An omniscient teacher has access to the student's feature space, model, loss function and optimization algorithm. In specific, omniscient teacher's $(x,y)$ and student's $(\xtil,\ytil)$ share the same representation space, and teacher's optimal model $v^*$ is also the same as student's optimal model $w^*$.
\vspace{-2.65mm}
\subsection{Intuition and teaching algorithm}
\vspace{-1.1mm}
In order to gain intuition on how to make the student model converge faster, we will start with looking into the difference between the current student parameter and the teacher parameter $w^*$ during each iteration:
\begin{equation}
\footnotesize
\begin{aligned}\label{graddecomp}
  &\nbr{w^{t+1} - w^\ast}_2^2= \nbr{w^t - \eta_t\, \frac{\partial \ell(\inner{w}{x},y)}{\partial w} - w^\ast}_2^2 \\[-1.4mm]
  =& \nbr{w^t - w^\ast}_2^2 + \eta_t^2 \underbrace{\nbr{\frac{\partial \ell(\inner{w^t}{x},y)}{\partial w^t}}_2^2}_{T_1(x,y|w^t): \text{Difficulty of an example $(x,y)$}} \\[-0.8mm]
  &- 2\eta_t \underbrace{\inner{w^t - w^\ast}{\frac{\partial \ell(\inner{w^t}{x},y)}{\partial w^t}}}_{T_2(x,y|w^t): \text{Usefulness of an example $(x,y)$}}
\end{aligned}
\end{equation}
Based on the decomposition of the parameter error, the teacher aims to choose a particular example $(x,y)$ such that $\thickmuskip=2mu \medmuskip=2mu \|w^{t+1} - w^\ast\|_2^2$ is most reduced compared to $\thickmuskip=2mu \medmuskip=2mu \|w^t - w^\ast\|_2^2$ from the last iteration. Thus the general strategy for the teacher is to choose an example $(x,y)$, such that $\thickmuskip=2mu \medmuskip=2mu \eta_t^2 T_1 - 2\eta_t T_2$ is minimized in the $t$-th iteration:
\begin{equation}\label{T1T2_min}
\begin{aligned}
\argmin_{x\in\Xcal,y\in\Ycal}\eta_t^2 T_1(x,y|w^t) - 2\eta_t T_2(x,y|w^t).
\end{aligned}
\end{equation}
The teaching algorithm of omniscient teacher is summarized in Alg.1. The smallest value of $\thickmuskip=2mu \medmuskip=2mu \eta_t^2 T_1 - 2\eta_t T_2$ is $-\|w^t-w^*\|_2^2$. If the teacher achieves this, it means that we have reached the teaching goal after this iteration. However, it usually cannot be done in just one iteration, because of the limitation of teacher's capability to provide examples. $T_1$ and $T_2$ have some nice intuitive interpretations:
\par
{\bf Difficulty of an example.} $T_1$ quantifies the difficulty level of an example. This interpretation for different loss functions becomes especially clear when the data lives on the surface of a sphere,~\ie, $\thickmuskip=2mu \medmuskip=2mu \nbr{x} = 1$. For instance,
\vspace{-2.65mm}
\begin{itemize}[leftmargin=*,nosep,nolistsep]
  \item For linear regression, $\thickmuskip=2mu \medmuskip=2mu T_1 = (\inner{w}{x} - y)^2$. The larger the norm of gradient is, the more difficult the example is.
  \item For logistic regression, we have $\thickmuskip=2mu \medmuskip=2mu T_1 = \|\frac{1}{1+\exp(y\inner{w}{x})}\|_2^2$. We know that $\frac{1}{1+\exp(y\inner{w}{x})}$ is the probability of predicting the wrong label. The larger the number is, the more difficult the example is.
  \item For support vector machines, we have $\thickmuskip=2mu \medmuskip=2mu T_1 = \frac{1}{2}(\textnormal{sign}(1-y\inner{w}{x})+1) $. Different from above losses, the hinge loss has a threshold to identify the difficulty of examples. While the example is difficult enough, it will produce $1$. Otherwise it is $0$.
\end{itemize}
\vspace{-2.65mm}
Interestingly, the difficulty level is not related to the teacher $w^*$, but is based on the current parameters of the learner $w^t$. From another perspective, the difficulty level can also be interpreted as the information that an example carries. Essentially, a difficult example is usually more informative. In such sense, our difficulty level has similar interpretation to curriculum learning, but with different expression.
\par
{\bf Usefulness of an example.} $T_2$ quantifies the usefulness of an example. Concretely, $T_2$ is the correlation between discrepancy $w^t - w^\ast$ and the information (difficulty) of an example. If the information of the example has large correlation with the discrepancy, it means that this example is very useful in this teaching iteration.
\par
{\bf Trade-off.} Eq.\eqref{T1T2_min} aims to minimize the difficulty level $T_1$ and maximize the usefulness $T_2$. In other word, the teacher always prefers easy but useful examples. When the learning rate is large, $T_1$ term plays a more important role. When learning rate is small, $T_2$ term plays a more important role. This suggests that initially the teacher should choose easier examples to feed into the student model, and later on the teacher should choose examples to focus more on reducing the discrepancy between $w^t - w^\ast$. Such examples are very likely the difficult ones. Even if the learning rate is fixed, the gradient $\nabla_{w}\ell$ is usually large for a convex loss function at the beginning, so reducing the difficulty level (choosing easy examples) is more important. While near the optimum, the gradient $\nabla_{w}\ell$ is usually small, so $T_2$ becomes more important. It is also likely to choose difficult examples. It has nice connection with curriculum learning (easy example first and
difficult later) and boosting (gradually focus on difficult examples).
\vspace{-2.4mm}
\subsection{Teaching monotonicity and universal speedup}
\vspace{-0.9mm}
Can the omniscient teacher always do better than a teacher who feed random examples to the student (in terms of convergence)? In this section, we identify generic conditions under which we can guarantee that the iterative teaching algorithm always perform better than random teacher.
\vspace{-1.5mm}
\begin{definition}[Teaching Volume]
For a specific loss function $\ell$, we first define a teaching volume function $TV(w)$ with model parameter $w$ as
\begin{equation}
\footnotesize
\begin{aligned}
TV(w)=\max_{x\in\mathcal{X},y\in\mathcal{Y}}\{-\eta_t^2 T_1(x,y|w) + 2\eta_t T_2(x,y|w)\}
\end{aligned}
\end{equation}
\end{definition}
\vspace{-2.3mm}
\begin{theorem}[Teaching Monotonicity]\label{thm:fast_teaching}
Given a training set $\mathcal{X}$ and a loss function $\ell$,  if the inequality
\begin{equation}
\footnotesize
\begin{aligned}
\nbr{w_1-w^*}^2-TV(w_1)\leq \nbr{w_2-w^*}^2 - TV(w_2)
\end{aligned}
\end{equation}
holds for any $w_1,w_2$ that satisfy $\thickmuskip=2mu \medmuskip=2mu \|w_1-w^*\|^2\leq\|w_2-w^*\|^2$, then with the same parameter initialization and learning rate, the omniscient teacher can always converge not slower than random teacher.
\end{theorem}
\vspace{-1.8mm}
The teaching volume represents the teacher's teaching effort in this iteration, so $\thickmuskip=0mu \medmuskip=0mu \|w^t-w^*\|^2-TV(w^t)$ characterizes the remaining teaching effort needed to achieve the teaching goal after iteration $t$. Theorem \ref{thm:fast_teaching} says that for a loss function and a training set, if the remaining teaching effort is monotonically decreasing while the model parameter gets closer to the optimum, we can guarantee that the omniscient teacher can always converge not slower than random teacher. It is a sufficient condition for loss functions to achieve faster convergence than SGD. For example, the square loss satisfies the condition with certain training set:
\vspace{-1.9mm}
\begin{proposition}\label{prop:square_loss}
The square loss satisfies the teaching monotonicity condition given the training set $\{x|\|x\|\leq R\}$.
\end{proposition}
\vspace{-3mm}
\subsection{Teaching capability and exponential speedup}
\vspace{-0.5mm}

The theorem in previous subsection insures that under certain conditions the omniscient teacher can always lead to faster convergence for the student model, but can there be exponential speedup? To this end, we introduce further assumptions of the ``richness'' of teaching examples, which we call teaching capability. We start from the ideal case, \ie, the synthesis-based omniscient teacher with hyperspherical feature space, and then, extend to real cases with the restrictions on teacher's knowledge domain, sampling scheme, and student information. We present specific teaching strategies in terms of teaching capability (strong to weak): synthesis, combination and (rescalable) pool.
\par
\textbf{Synthesis-based teaching}. In synthesis-based teaching, the teacher can provide any samples from
\begin{equation}
\footnotesize
\begin{aligned}
\Xcal &= \{x\in \RR^d, \nbr{x}\le R\}\\
  \Ycal &= \RR\text{ (Regression) or } \cbr{-1, 1}\text{ (Classification)}. \nonumber
\end{aligned}
\end{equation}
\vspace{-4.3mm}
\begin{theorem}[Exponential Synthesis-based Teaching]\label{thm:opt_synthesis}
For a synthesis-based omniscient teacher and a student with fixed learning rate $\thickmuskip=2mu \medmuskip=2mu \eta\neq0$, if the loss function $\ell(\cdot, \cdot)$ satisfies that for any $w\in \RR^d$, there exists $\thickmuskip=2mu \medmuskip=2mu \gamma\neq 0$, $\thickmuskip=2mu \medmuskip=2mu \abr{\gamma} \le \frac{R}{\nbr{w-w^\ast}}$ such that while $\thickmuskip=2mu \medmuskip=2mu \hat x = \gamma\rbr{w - w^\ast}$ and $\thickmuskip=2mu \medmuskip=2mu \hat y \in \Ycal$, we have

\vspace{-5mm}
\begin{equation*}
\footnotesize
\begin{aligned}
0<  \gamma\nabla_{\inner{w}{\hat x}} \ell\rbr{\inner{w}{\hat x}, \hat y}\le \frac{1}{\eta},
\end{aligned}
\end{equation*}

\vspace{-2mm}
then the student can learn an $\epsilon$-approximation of $w^\ast$ with $\Ocal(C_1^{\gamma, \eta}\log\frac{1}{\epsilon})$ samples. We call such loss function $\ell(\cdot, \cdot)$ exponentially teachable in synthesis-based teaching.
\end{theorem}
\vspace{-0.5mm}

The constant is $\thickmuskip=2mu \medmuskip=2mu C_1^{\gamma, \eta} = (\log\frac{1}{1-\eta\nu(\gamma)})^{-1}$ in which $\thickmuskip=2mu \medmuskip=2mu \nu(\gamma):=\min_{w, y} \gamma\nabla_{\inner{w}{\hat x}} \ell\rbr{\inner{w}{\hat x}, y} > 0$. $\nu(\gamma)$ is related to the convergence speed. Note that the sample complexity serves as the iterative teaching dimension corresponding to this particular teacher, student, algorithm and training data.

\vspace{-.7mm}
The sample complexity in iterative teaching is \emph{deterministic}, different from the high probability bounds of traditional sample complexity with random \iid samples or actively required samples. This is because the teacher provides the samples deterministically without noise in every iteration.
\par
\vspace{-.7mm}
The radius $R$ for $\Xcal$, which can be interpreted as the knowledge domain of the teacher, will affect the sample complexity by constraining the valid values of $\gamma$, and thus $C^{\gamma, \eta}_1$. For example, for absolute loss, if $R$ is large, such that $\thickmuskip=2mu \medmuskip=2mu\frac{1}{\eta}\leq \frac{R}{\nbr{w^0-w^\ast}}$, $\gamma$ can be set to $\frac{1}{\eta}$ and the $\nu(\gamma)$ will be $\frac{1}{\eta}$ in this case. Therefore, we have $\thickmuskip=2mu \medmuskip=2mu C^{\gamma, \eta}_1=0$, which means the student can learn with only one example (one iteration). However, if $\thickmuskip=2mu \medmuskip=2mu\frac{1}{\eta}> \frac{R}{\nbr{w^0-w^\ast}}$, we have $\thickmuskip=2mu \medmuskip=2mu C^{\gamma, \eta}_1 >0$, and the student can converge exponentially. The similar phenomenon appears in the square loss, hinge loss, and logistic loss. Refer to Appendix~\ref{appendix:proof} for details.
\par
\vspace{-.7mm}
The exponential synthesis-based teaching is closely related to Lipschitz smoothness and strong convexity of loss functions in the sense that the two regularities provide positive lower and upper bound for $\gamma\nabla_{\inner{w}{x}} \ell\rbr{\inner{w}{x}, y}$.
\vspace{-1.4mm}
\begin{proposition}\label{prop:general_teachable}
The Lipschitz smooth and strongly convex loss functions are exponentially teachable in synthesis-based teaching.
\end{proposition}
\vspace{-1.4mm}
The exponential synthesis-based teachability is a weaker condition compared to the strong convexity and Lipschitz smoothness. We can show that besides the Lipschitz smooth and strongly convex loss, there are some other loss functions, which are not strongly convex, but still are exponentially teachable in synthesis-based scenario, \eg, the hinge loss and logistic loss. Proofs are in Appendix~\ref{appendix:proof}.
\par
{\bf Combination-based teaching.} In this scenario, the teacher can provide examples from ($\alpha_i\in\mathbb{R}$)
\vspace{-0.5mm}
\begin{equation}
\footnotesize
\begin{aligned}
  \Xcal &= \big{\{} x | \|x\|\leq R, x = \sum_{i=1}^m\alpha_i x_i, x_i \in \Dcal \big{\}}, \Dcal = \cbr{x_1, \ldots, x_m} \\
  \Ycal &= \RR\text{ (Regression) or } \cbr{-1, 1}\text{ (Classification)} \nonumber
\end{aligned}
\end{equation}

{
\begin{algorithm}[t]\label{alg1}
\small
    \caption{The omniscient teacher}
    \begin{algorithmic}[1]
    \STATE Randomly initialize the student and teacher parameter $w^0$;
    \STATE Set $t = 1$ and the maximal iteration number $T$;
    \WHILE {$w^t$ has not converged or $t<T$}
    \STATE Solve the optimization (e.g., pool-based teaching):\\
    \begin{equation*}
    \setlength{\abovedisplayskip}{-2mm}
    \setlength{\belowdisplayskip}{1mm}
    \tiny
    \begin{aligned}
      (x^t,y^t) = &\argmin_{x\in\mathcal{X},y\in\mathcal{Y}}~\eta_t^2 \nbr{  \frac{\partial \ell\rbr{\inner{w^{t-1}}{x}, y}}{\partial w^{t-1}}}^2 \- \\[-1mm]
      &-2\eta_t \inner{w^{t-1} - w^\ast}{ \frac{\partial \ell\rbr{\inner{w^{t-1}}{ x},y}}{\partial w^{t-1}}}
    \end{aligned}
    \end{equation*}
    \STATE Use the selected example $(x^t,y^t)$ to perform the update:
    \begin{equation*}
    \setlength{\abovedisplayskip}{0.8mm}
    \setlength{\belowdisplayskip}{1mm}
    \tiny
    \begin{aligned}
      w^{t} = w^{t-1} - \eta_t\, \frac{\partial \ell\rbr{\inner{w^{t-1}}{x^t},y^t}}{\partial w^{t-1}}.
    \end{aligned}
    \end{equation*}
    \STATE $t\leftarrow t+1$
    \ENDWHILE
    \end{algorithmic}
\end{algorithm}
}

\par
\vspace{-3.5mm}
\begin{corollary}\label{cor:opt_combination}
For a combination-based omniscient teacher and a student with fixed learning rate $\eta\neq 0$ and initialization $w^0$, if the loss function is exponentially synthesis-based teachable and $\thickmuskip=2mu \medmuskip=2mu w^0-w^\ast \in \text{span}\rbr{\Dcal}$, the student can learn an $\epsilon$-approximation of $w^\ast$ with $\Ocal\rbr{C_1^{\gamma, \eta}\log\frac{1}{\epsilon}}$ samples.
\end{corollary}
\vspace{-1.5mm}
\par
Although the knowledge pool of teacher is more restricted compared to the synthesis-based scenario, with teacher's extra work to combine samples, the teacher can behave the same as the most knowledgable synthesis-based teacher.
\par
\vspace{-0.2mm}
{\bf Rescalable pool-based teaching.} This scenario is further restricted in both knowledge pool and the effort to prepare samples. The teacher can provide examples from $\thickmuskip=2mu \medmuskip=2mu \Xcal\times\Ycal$:
\begin{equation}
\footnotesize
\begin{aligned}
\Xcal & = \{ x | \|x\|\leq R,  x = \gamma x_i, x_i \in \Dcal, \gamma\in\RR \}, \Dcal= \{ x_1, \ldots\}\\
\Ycal &= \RR\text{ (Regression) or } \cbr{-1, 1}\text{ (Classification)} \nonumber
\end{aligned}
\end{equation}
\vspace{-3.6mm}

In such scenario, we cannot get arbitrary direction rather than the samples from the candidate pool. Therefore, to achieve the exponential improvement, the candidate pool should contain rich enough directions. To characterize the richness in finite case, we define the \emph{pool volume} as

\vspace{-0.8mm}

\begin{definition}[Pool Volume]
Given the training example pool $\Xcal\in \RR^d$, the volume of $\Xcal$ is defined as
\begin{equation*}
\footnotesize
\begin{aligned}
\Vcal(\Xcal) := \min_{w\in \text{span}\rbr{\Dcal}}\max_{x\in\Xcal}\frac{\inner{w}{x}}{\nbr{w}^2}.
\end{aligned}
\end{equation*}
\end{definition}
Obviously, for the candidate pool of the synthesis-based teacher, we have $\thickmuskip=2mu \medmuskip=2mu \Vcal(\Xcal) = 1$. In general, for finite candidate pool, the pool volume is $\thickmuskip=2mu \medmuskip=2mu 0<\Vcal(\Xcal)< 1$.
\par
\vspace{-1.45mm}
\begin{theorem}\label{thm:opt_pool}
For a rescalable pool-based omniscient teacher and a student with fixed learning rate $\thickmuskip=2mu \medmuskip=2mu \eta\neq 0$ and initialization $w^0$, if for any $\thickmuskip=2mu \medmuskip=2mu w\in \RR^d$, $\thickmuskip=2mu \medmuskip=2mu w\not\perp w^\ast$ and $\thickmuskip=2mu \medmuskip=2mu w^0-w^\ast \in \text{span}\rbr{\Dcal}$, there exists $\thickmuskip=2mu \medmuskip=2mu \cbr{x, y}\in \Xcal \times \Ycal$ and $\gamma$ such that while $\thickmuskip=2mu \medmuskip=2mu \hat x = \frac{\gamma\nbr{w - w^\ast}}{\nbr{x}}x, \hat y =y$, we have

\vspace{-4.7mm}
\begin{equation*}
\footnotesize
\begin{aligned}
0<  \gamma\nabla_{\inner{w}{\hat x}} \ell\rbr{\inner{w}{\hat x}, \hat y}< \frac{{2\Vcal(\Xcal)}}{\eta},
\end{aligned}
\end{equation*}
\vspace{-4.7mm}

then the student can learn an $\epsilon$-approximation of $w^*$ with $\Ocal(C_2^{\eta, \gamma, \Vcal(\Xcal)}\log\frac{1}{\epsilon})$ samples. We say such loss function is exponentially teachable in rescalable pool-based teaching.
\end{theorem}
\vspace{-2.5mm}

The pool volume plays a vital role in pool-based teaching. It not only affects the existence of $\gamma$ and $\cbr{\hat x, \hat y}$ to satisfy the conditions, but also changes the convergence rate. While $\Vcal(\Xcal)$ increases, $C_2^{\eta, \gamma, \Vcal(\Xcal)}$ will decrease, yielding smaller sample complexity.
With $\thickmuskip=2mu \medmuskip=2mu \Vcal(\Xcal)< 1$, the rescalable pool-based teaching requires more samples than the synthesis-based teaching. As $\Vcal(\Xcal)$ increases to $1$, the candidate pool becomes $\thickmuskip=2mu \medmuskip=2mu \cbr{x\in \RR^d, \nbr{x}\le R}$ and $C_2^{\eta, \gamma, \Vcal(\Xcal)}$ approaches to $C_1^{\gamma, \eta} $. Then the convergence speed of rescalable pool-based teaching approaches to the synthesis/combination-based teaching.
\par

\vspace{-2mm}
\section{Teaching by a less informative teacher}
\vspace{-.1mm}
To make the teacher model useful in practice, we further design two less informative teacher model that requires less and less information from the student.
\vspace{-2.4mm}
\subsection{The surrogate teacher}
\vspace{-1mm}
Suppose we can only query the function output from the learned $\inner{w^t}{x}$, but we can not directly access $w^t$. How can we choose the example? In this case we propose to make use of the the convexity of the loss function. That is
\begin{equation}
\footnotesize
\begin{aligned}
\inner{w^t - w^\ast}{\frac{\partial \ell(\inner{w^t}{x},y)}{\partial w^t}x}\geqslant \ell(\inner{w^t}{x}, y) - \ell(\inner{w^*}{x}, y).
\end{aligned}
\end{equation}
Taking the pool-based teaching as an example, we can instead optimize the following surrogate loss function:
\begin{equation}\label{eq:poollowbou}
\footnotesize
\begin{aligned}
(x^t, y^t) = &\argmin_{\{x,y\}\in\Xcal}~ \eta_t^2 \nbr{\frac{\partial \ell(\inner{w^t}{x},y)}{\partial w^t}}_2^2 \\
&- 2\eta_t \rbr{\ell(\inner{w^t}{x}, y) - \ell(\inner{w^*}{x}, y)}
\end{aligned}
\end{equation}
by replacing $\inner{w^t - w^\ast}{\frac{\partial \ell(\inner{w^t}{x},y)}{\partial w^t}}$ with its lower bound. The advantage of this approach is that the teacher only need to query the learner for the function output $\inner{w^t}{x}$ to choose the example, without the need to access the learner parameter $w^t$ directly. Furthermore, after noticing that in this formulation, the teacher makes prediction via inner products, we find that the surrogate teacher can also be applied to the scenario where the teacher and the student use different feature spaces by further replacing $\thickmuskip=2mu \medmuskip=2mu \rbr{\ell(\inner{w^t}{x}, y) - \ell(\inner{w^*}{x}, y)}$ with $\thickmuskip=2mu \medmuskip=2mu \rbr{\ell(\inner{w^t}{x}, y) - \ell(\inner{v^*}{\widetilde{x}}, y)}$. With this modification, we can provide examples without using information about $w^*$. The performance of the surrogate teacher largely depends on the tightness of such convexity lower bound.

\vspace{-0.6mm}
\subsection{The imitation teacher}
\vspace{-0.7mm}
When the teacher and the student have different feature spaces, this teaching setting will be much closer to practice than all the previous settings and also more challenging. To this end, we present an imitation teacher who learns to imitate the inner product output $\inner{w^t}{x}$ of the student model and simultaneously choose examples in teacher's own feature space. The teacher can possibly use active learning to imitate the student's $\inner{w^t}{x}$. In this imitation, the student model stays unchanged and the teacher model could update itself via multiple queries to the student (input an example and see the inner product output of the student). We propose a more simple and straightforward imitation teacher (Alg. 2) which works in a way similar to stochastic mirror descent  \cite{nemirovski2009robust,hall2013online}. In specific, the teacher first learns to approximate the student's  $\inner{w^t}{x}$ with the following iterative update:
\begin{equation}
\footnotesize
\thickmuskip=2mu \medmuskip=2mu v^{t+1}=v^{t}-\eta_v\rbr{\inner{v^t}{x}-\inner{w^t}{x}}x
\end{equation}
where $\eta_v$ is the learning rate for the update. Then we use $v^{t+1}$ to perform the example synthesis or selection in teacher's own feature space. We summarize this simple yet effective imitation teacher model in Alg. 2.

{
\begin{algorithm}[t]\label{alg_it2}
\small
    \caption{The imitation teacher}
    \begin{algorithmic}[1]
    \STATE Randomly initialize the student parameter $w^0$ and the teacher parameter $v^0$; Randomly select a training sample $(x^0,y^0)$;
    \STATE Set $t = 1$ and the maximal iteration number $T$;
    \WHILE {$w^t$ has not converged or $t<T$}
    \STATE Perform the update:\\
    \begin{equation*}
    \setlength{\abovedisplayshortskip}{-1.6mm}
    \setlength{\belowdisplayshortskip}{1mm}
    \tiny
    \begin{aligned}
    v^{t}=v^{t-1}-\eta_v\rbr{\inner{v^{t-1}}{x^{t-1}}-\inner{w^t}{x^{t-1}}}x^{t-1}
    \end{aligned}.
    \end{equation*}
    \STATE Solve the optimization (e.g., pool-based teaching):\\
    \begin{equation*}
    \setlength{\abovedisplayshortskip}{-1.6mm}
    \setlength{\belowdisplayshortskip}{1mm}
    \tiny
    \begin{aligned}
      (x^t,y^t) = &\argmin_{x\in\mathcal{X},y\in\mathcal{Y}}~\eta_t^2 \nbr{ \frac{\partial \ell\rbr{\inner{w^t}{ x}, y}}{\partial v^t}}^2 \- \\[-1mm]
      &-2\eta_t \inner{v^t - v^\ast}{ \frac{\partial \ell\rbr{\inner{v^t}{ x},y}}{\partial v^t}}
    \end{aligned}.
    \end{equation*}
    \STATE Provide the selected example $(x^t,y^t)$ for the student to perform the update ;
    \begin{equation*}
    \setlength{\abovedisplayskip}{-0.5mm}
    \setlength{\belowdisplayskip}{1mm}
    \tiny
    \begin{aligned}
      w^{t+1} = w^t - \eta_t\, \frac{\partial \ell\rbr{\inner{w^t}{x},y}}{\partial w}.
    \end{aligned}
    \end{equation*}
    \STATE $t\leftarrow t+1$
    \ENDWHILE
    \end{algorithmic}
\end{algorithm}
}

\vspace{-2mm}
\section{Discussion}
\vspace{-0.2mm}
\textbf{Optimality of the teacher model.} For arbitary loss function, the optimal teacher model for a student model should find the training example sequence to achieve the fastest possible convergence. Exhaustively finding such example sequence is computational impossible. For example, there are $n^T$ possible training sequences ($T$ is the iteration number) for $n$-size pool-based teaching. As a results, we need to make use of the properties of loss function to design the teacher model. The proposed teacher models are not necessarily optimal, but they are good enough under some conditions for loss function, student model and training data.
\par
\textbf{Theoretical aspects of the teacher model.} The theoretical study of the teacher model includes finding the conditions for the loss function and training data such that the teacher model is optimal, or achieves provable faster convergence rate, or provably converges faster than the random teacher. We desire these conditions to be sufficient and necessary, but sometimes sufficient conditions suffice in practice. For different student models, the theoretical analysis may be different and we merely consider stochastic gradient learner here. There are still lots of optimization algorithms that can be considered. Besides, our teacher models are not necessarily the best, so it is also important to come up with better teacher models with provable guarantees. Although our paper mainly focuses on the fixed learning rate, our results are still applicable for the dynamic learning rate. However, the teacher should be more powerful in synthesizing or choosing examples ($R$ should be larger than fixed learning rate case). In human teaching, it actually makes sense because while teaching a student who learns knowledge with dynamic speed, the teacher should be more powerful so that the student consistently learn fast.
\par
\textbf{Practical aspects of the teacher model.} In practice, we usually want the teacher model to be less and less informative to the student model, scalable to large datasets, efficient to compute. How to make the teacher model scalable, efficient and less informative remains open challenges.

\vspace{-2.5mm}
\section{Experiments}
\vspace{-1mm}
\subsection{Experimental details}
\vspace{-1mm}

\textbf{Performance metric.} We use three metric to evaluate the convergence performance: objective value w.r.t. the training set, difference between $w^t$ and $w^*$ ($\nbr{w^t-w^*}_2$), and the classification accuracy on testing set. 
\par
\vspace{-2mm}
\textbf{Parameters and setup.} Detailed experimental setup is given in Appendix \ref{appendix:exp}. We mostly evaluate the practical pool-based teaching (without rescaling). We evaluate the different teaching strategies in Appendix \ref{appendix:diffstra}, and give more experiments on spherical data (Appendix \ref{appendix:spherical}) and infant ego-centric visual data (Appendix \ref{appendix:infant}). For fairness, learning rates for all methods are the same.

\begin{figure*}[t]
  \centering
  \footnotesize
  \includegraphics[width=6.35in]{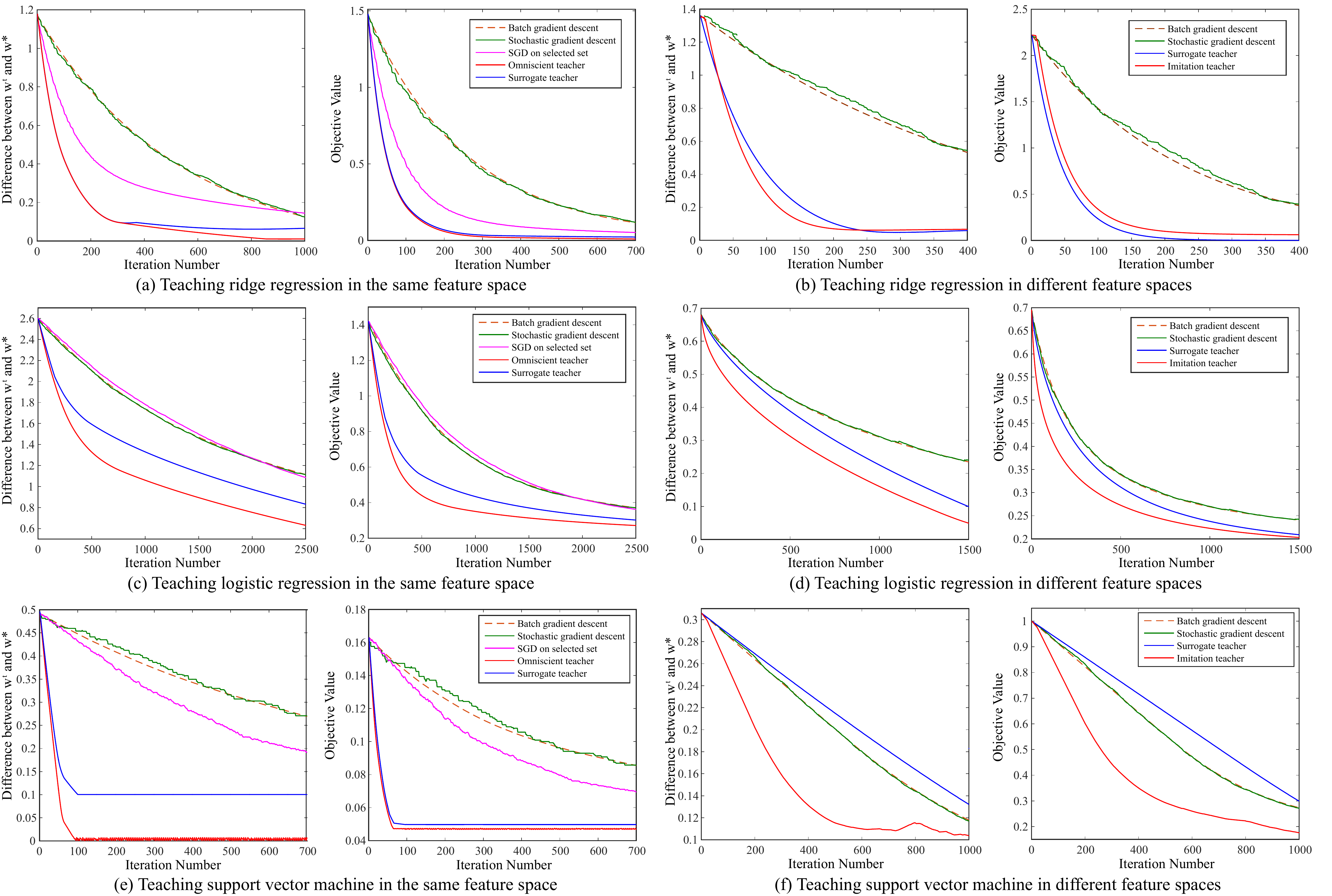}
  \vspace{-2mm}
  \caption{Convergence results on Gaussian distributed data.}\label{exp1}
  \vspace{-3.5mm}
\end{figure*}

\vspace{-2mm}
\subsection{Teaching linear models on Gaussian data}
\vspace{-0.2mm}
\par
\label{gaussian}
This experiment explores the convergence of three typical linear models: ridge regression (RR), logistic regression (LR) and support vector machine (SVM) on Gaussian data. Note that SGD on selected set is to run SGD on the union of all samples selected by the omniscient teacher. For the scenario of different feature spaces, we use a random orthogonal projection matrix to generate the teacher's feature space from student's. All teachers use pool-based teaching strategy. For fair comparisons, we use the same random initialization and the same learning rate.
\begin{figure}[h]
  \centering
  \footnotesize
  \includegraphics[width=2.44in]{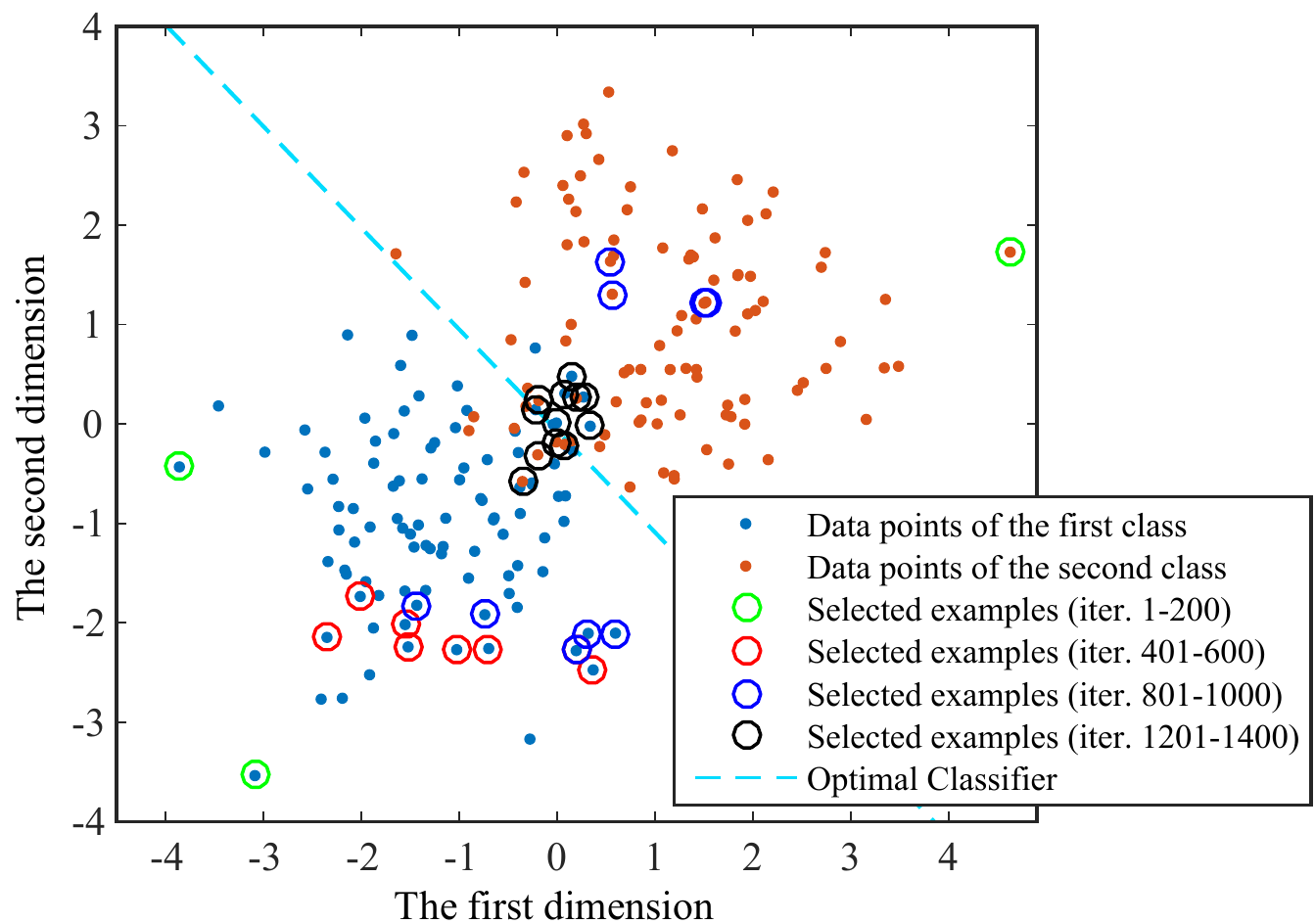}
  \vspace{-3mm}
  \caption{The examples selected by omniscient teacher for logistic regression on 2D binary-class Gaussian data.}\label{2d}
  \vspace{1mm}
\end{figure}
\vspace{-1.5mm}
\par
\textbf{Teaching in the same feature space.} The results in Fig. \ref{exp1} show that the learner can converge much faster using the example provided by the teacher, showing the effectiveness of our teaching models. As expected, we find that the omniscient teacher consistently achieves faster convergence than the surrogate teacher who has no access to $w$. It is because the omniscient teacher always has more information about the learner. More interestingly, our guiding algorithms also consistently outperform SGD on the selected set, showing that the order of inputting training samples matters.
\par
\vspace{-1.3mm}
\textbf{Teaching in different feature spaces.} It is a more practical scenario that teacher and student use different feature spaces. While the omniscient teacher model is no longer applicable here, we teach the student model using the surrogate teacher and the imitation teacher. While the feature spaces are totally different, it can be expected that there will be a mismatch gap between the teacher model parameter and the student model parameter. Even in such a challenging scenario, the experimental results show that our teacher model still outperforms the conventional SGD and batch GD in most cases. One can observe that the surrogate teacher performs poorly in the SVM, which may be caused by the tightness of the approximated lower bound of the $T_2$ term. Compared to the surrogate teacher, the imitation teacher is more stable and consistently improves the convergence in all three linear models.

\vspace{-1.5mm}
\subsection{Teaching Linear Classifiers on MNIST Dataset}
\vspace{-.5mm}
We further evaluate our teacher models on MNIST dataset. We use 24D random features to classify the digits (0/1, 3/5 as examples). We generate the teacher's features using a random projection matrix from the original 24D student's features. Note that, omniscient teacher and surrogate teacher (same space) assume the teacher uses the student's feature space, while surrogate teacher (different space) and imitation teacher assume the teacher uses its own space. From Fig. \ref{mnist_exp}, one can observe that all these teacher model produces significant convergence speedup. We can see that the omniscient teacher converges fastest as expected. Interestingly, our imitation teacher achieves very similar convergence speedup to the omniscient teacher under the condition that the teacher does not know the student's feature space. In Fig.\ref{mnist_vis}, we also show some examples of teacher's selected digit images (0/1 as examples) and find that the teacher tends to select easy example at the beginning and gradually shift the focus to difficult examples. This also has the intrinsic connections with the curriculum learning.

\begin{figure}[t]
  \centering
  \footnotesize
  \includegraphics[width=3.025in]{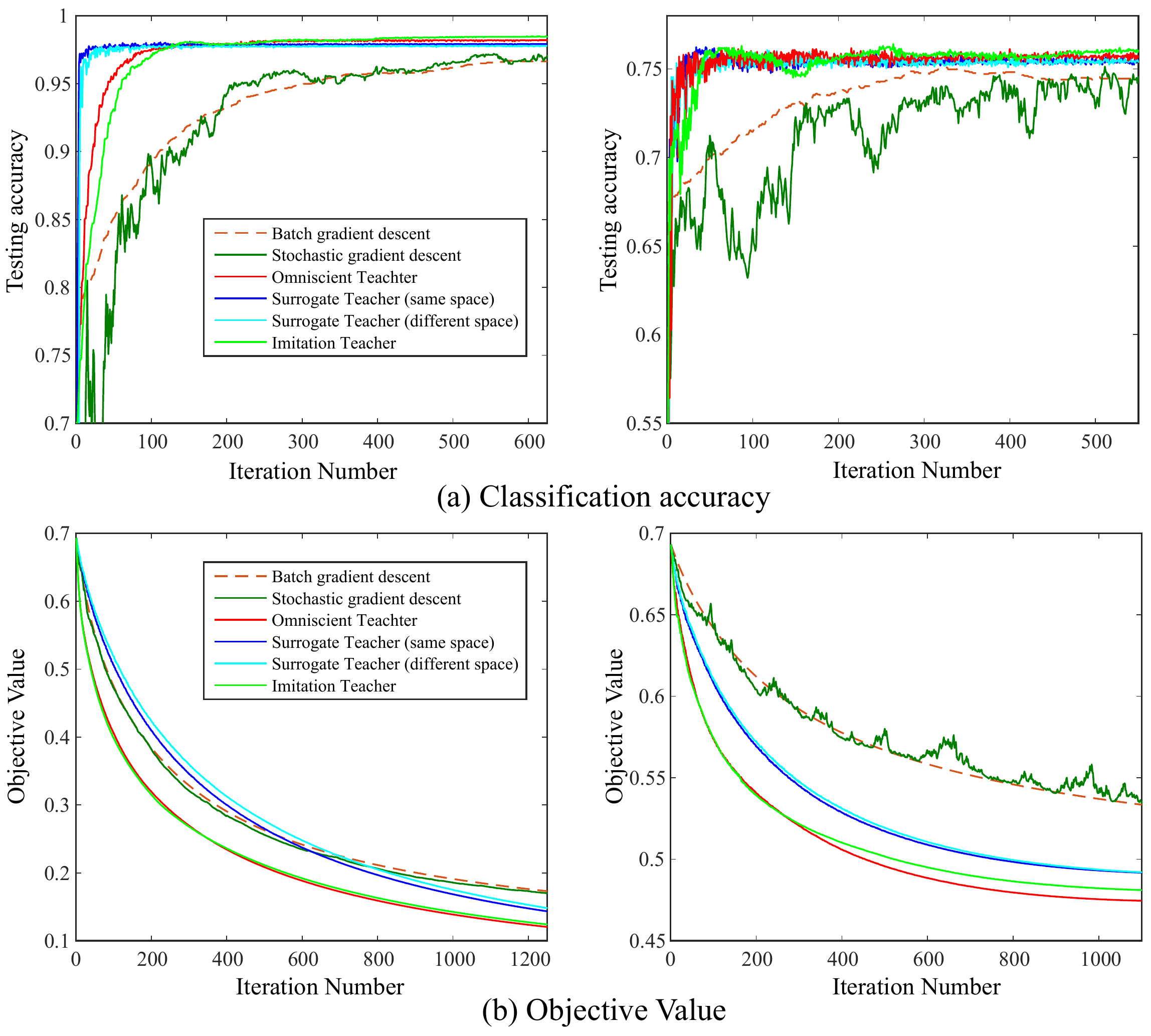}
  \vspace{-3.8mm}
  \caption{Teaching logistic regression on MNIST dataset. Left column: 0/1 classification. Right column: 3/5 classification}\label{mnist_exp}
  \vspace{-1.85mm}
\end{figure}

\begin{figure}[t]
  \centering
  \footnotesize
  \includegraphics[width=3in]{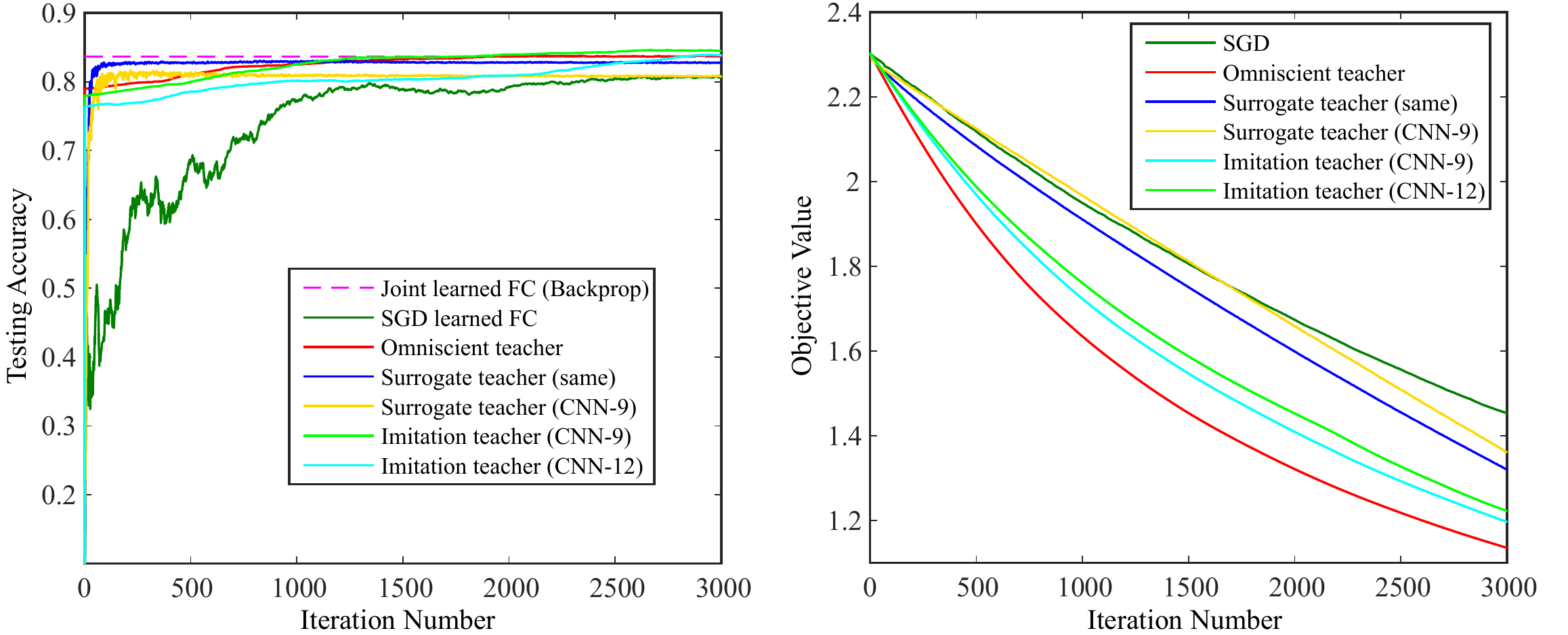}
  \vspace{-3.5mm}
  \caption{Teaching fully connected layers of CNNs on CIFAR-10. Left: testing accuracy. Right: training objective value.}\label{cifar10_exp}
  \vspace{1.5mm}
\end{figure}

\begin{figure}[t]
  \centering
  \footnotesize
  \includegraphics[width=3in]{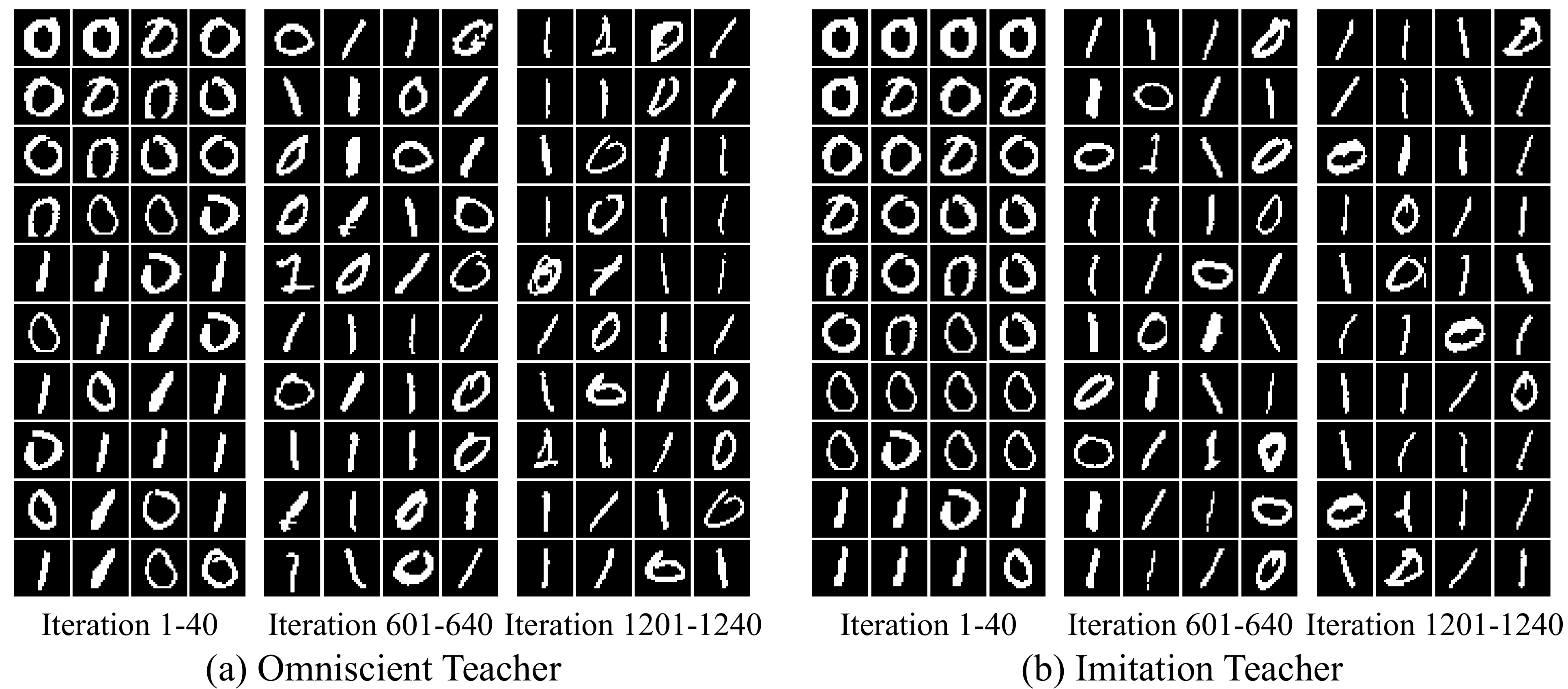}
  \vspace{-3.5mm}
  \caption{Some selected training examples on MNIST.}\label{mnist_vis}
  \vspace{-2.5mm}
\end{figure}

\begin{figure}[t]
  \centering
  \footnotesize
  \includegraphics[width=3.1in]{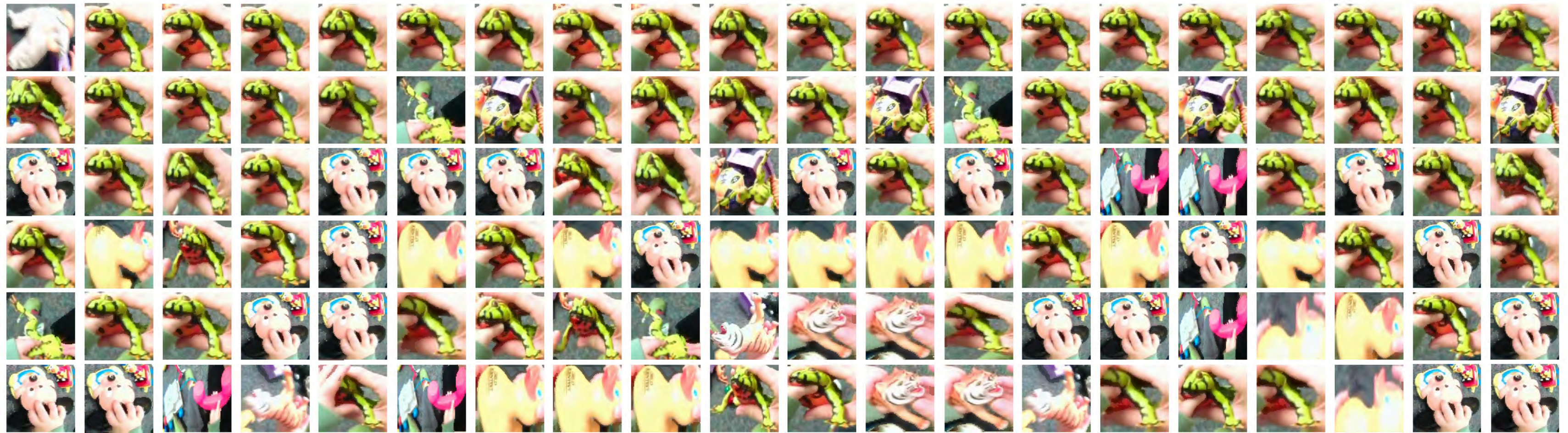}
  \vspace{-2.7mm}
  \caption{Selected training examples by the omniscient teacher on ego-centric data of infants. (The examples are visualized every 100 iteration, with left-to-right and top-to-bottom ordering)}\label{infant_vis}
  \vspace{1.1mm}
\end{figure}

\vspace{-1.9mm}
\subsection{Teaching Fully Connected Layers in CNNs}
\vspace{-0.8mm}
We extend our teacher models from binary classification to multi-class classification. The teacher models are used to teach the final fully connected (FC) layers in convolutional neural network on CIFAR-10. We first train three baseline CNNs (6/9/12 convolution layers, detailed configuration is in Appendix \ref{appendix:exp}) on CIFAR-10 without data augmentation and obtain the 83.5\%, 86.1\%, 87.2\% accuracy. First, we applied the omniscient teacher and the surrogate teacher to the CNN-6 student using the optimal FC layer from the joint backprop training. It is essentially to teach the FC layer in the same feature space. Second, we applied the surrogate teacher and the imitation teacher to the CNN-6 student using the parameters of optimal FC layers from CNN-9 and CNN-12. It is to teach the FC layer in different feature spaces. More interestingly, this different feature space may not necessarily have an invertible one-to-one mapping, but we could still observe convergence speedup using our teacher models. From Fig. \ref{cifar10_exp}, we can see that all the teacher models produces very fast convergence in terms of testing accuracy. Our teacher models can even produces better testing accuracy than the backprop-learned FC layer. For objective value, the omniscient teacher shows the largest convergence speedup, and the imitation teacher performs slightly worse but still much better than the SGD.

\vspace{-1.95mm}
\subsection{Teaching on ego-centric visual data of infants}
\vspace{-0.85mm}
Using our teaching model, we analyze cropped object instances obtained from ego-centric video of an infant playing with toys~\cite{yurovsky2013}. Full detailed settings and results are in Appendix \ref{appendix:infant}. The results in Fig. \ref{infant_vis} demonstrate a strong qualitative agreement between the training examples selected by the omniscient teacher and the order of examples received by a child in a naturalistic play environment. In both cases, the learner experiences extended bouts of viewing the same object. In contrast, the standard SGD learner receives random inputs. Our convergence results demonstrate that the learner converges significantly faster when receiving similar inputs to the child. Previous works have documented the unique temporal structure of the image examples that a child receives during object play~\cite{bambach2016,pereira2014}. We believe these are the first results demonstrating that similar orderings can be obtained via a machine teaching approach.

\vspace{-2.7mm}
\section{Concluding Remarks}
\vspace{-1.2mm}
The paper proposes an iterative machine teaching framework. We elaborate the settings of the framework, and then study two important properties: teaching monotonicity and teaching capability. Based on the framework, we propose three teacher models for gradient learners, and give theoretical analysis for the learner to provably achieve fast convergence. Our theoretical findings are verified by experiments.

\section*{Acknowledgement}
We would like to sincerely thank all the reviewers and Prof. Xiaojin Zhu for the valuable suggestions to improve the paper, Dan Yurovsky and Charlotte Wozniak for their help in collecting the dataset of children's visual inputs during object learning, and Qian Shao for help with the annotations. This project was supported in part by NSF IIS-1218749, NIH BIGDATA 1R01GM108341, NSF CAREER IIS-1350983, NSF IIS-1639792 EAGER, ONR N00014-15-1-2340, NSF Awards (BCS-1524565, BCS-1523982, and IIS-1320348) Nvidia and Intel. In addition, this work was partially supported by the Indiana University Areas of Emergent Research initiative in Learning: Brains, Machines, Children.

\bibliography{guiding}

\begin{thebibliography}{27}
\providecommand{\natexlab}[1]{#1}
\providecommand{\url}[1]{\texttt{#1}}
\expandafter\ifx\csname urlstyle\endcsname\relax
  \providecommand{\doi}[1]{doi: #1}\else
  \providecommand{\doi}{doi: \begingroup \urlstyle{rm}\Url}\fi

\bibitem[Alfeld et~al.(2016)Alfeld, Zhu, and Barford]{alfeld2016data}
Alfeld, Scott, Zhu, Xiaojin, and Barford, Paul.
\newblock Data poisoning attacks against autoregressive models.
\newblock In \emph{AAAI}, pp.\  1452--1458, 2016.

\bibitem[Alfeld et~al.(2017)Alfeld, Zhu, and Barford]{alfeld2017explicit}
Alfeld, Scott, Zhu, Xiaojin, and Barford, Paul.
\newblock Explicit defense actions against test-set attacks.
\newblock In \emph{AAAI}, 2017.

\bibitem[Ba \& Caruana(2014)Ba and Caruana]{ba2014deep}
Ba, Jimmy and Caruana, Rich.
\newblock Do deep nets really need to be deep?
\newblock In \emph{Advances in neural information processing systems}, pp.\
  2654--2662, 2014.

\bibitem[Balcan et~al.(2010)Balcan, Hanneke, and Vaughan]{balcan2010true}
Balcan, Maria-Florina, Hanneke, Steve, and Vaughan, Jennifer~Wortman.
\newblock The true sample complexity of active learning.
\newblock \emph{Machine learning}, 80\penalty0 (2-3):\penalty0 111--139, 2010.

\bibitem[Bambach et~al.(2016)Bambach, Crandall, Smith, and Yu]{bambach2016}
Bambach, Sven, Crandall, David~J, Smith, Linda~B, and Yu, Chen.
\newblock {Active Viewing in Toddlers Facilitates Visual Object Learning: An
  Egocentric Vision Approach}.
\newblock \emph{Proceedings of the 38th Annual Meeting of the Cognitive Science
  Society}, 2016.

\bibitem[Bengio et~al.(2009)Bengio, Louradour, Collobert, and
  Weston]{bengio2009curriculum}
Bengio, Yoshua, Louradour, Jerome, Collobert, Ronan, and Weston, Jason.
\newblock Curriculum learning.
\newblock In \emph{ICML}, 2009.

\bibitem[Bucila et~al.(2006)Bucila, Caruana, and
  Niculescu-Mizil]{bucila2006model}
Bucila, Cristian, Caruana, Rich, and Niculescu-Mizil, Alexandru.
\newblock Model compression.
\newblock In \emph{Proceedings of the 12th ACM SIGKDD international conference
  on Knowledge discovery and data mining}, pp.\  535--541. ACM, 2006.

\bibitem[Cakmak \& Thomaz(2014)Cakmak and Thomaz]{cakmak2014eliciting}
Cakmak, Maya and Thomaz, Andrea~L.
\newblock Eliciting good teaching from humans for machine learners.
\newblock \emph{Artificial Intelligence}, 217:\penalty0 198--215, 2014.

\bibitem[Doliwa et~al.(2014)Doliwa, Fan, Simon, and
  Zilles]{doliwa2014recursive}
Doliwa, Thorsten, Fan, Gaojian, Simon, Hans~Ulrich, and Zilles, Sandra.
\newblock Recursive teaching dimension, vc-dimension and sample compression.
\newblock \emph{Journal of Machine Learning Research}, 15\penalty0
  (1):\penalty0 3107--3131, 2014.

\bibitem[Goldman \& Kearns(1995)Goldman and Kearns]{goldman1995complexity}
Goldman, Sally~A and Kearns, Michael~J.
\newblock On the complexity of teaching.
\newblock \emph{Journal of Computer and System Sciences}, 50\penalty0
  (1):\penalty0 20--31, 1995.

\bibitem[Hall \& Willett(2013)Hall and Willett]{hall2013online}
Hall, Eric~C and Willett, Rebecca~M.
\newblock Online optimization in dynamic environments.
\newblock \emph{arXiv preprint arXiv:1307.5944}, 2013.

\bibitem[Han et~al.(2015)Han, Mao, and Dally]{han2015deep}
Han, Song, Mao, Huizi, and Dally, William~J.
\newblock Deep compression: Compressing deep neural networks with pruning,
  trained quantization and huffman coding.
\newblock \emph{arXiv preprint arXiv:1510.00149}, 2015.

\bibitem[Hinton et~al.(2015)Hinton, Vinyals, and Dean]{hinton2015distilling}
Hinton, Geoffrey, Vinyals, Oriol, and Dean, Jeff.
\newblock Distilling the knowledge in a neural network.
\newblock \emph{arXiv preprint arXiv:1503.02531}, 2015.

\bibitem[Johns et~al.(2015)Johns, Mac~Aodha, and Brostow]{JohnsCVPR2015}
Johns, Edward, Mac~Aodha, Oisin, and Brostow, Gabriel~J.
\newblock Becoming the expert - interactive multi-class machine teaching.
\newblock In \emph{CVPR}, 2015.

\bibitem[Khan et~al.(2011)Khan, Mutlu, and Zhu]{khan2011humans}
Khan, Faisal, Mutlu, Bilge, and Zhu, Xiaojin.
\newblock How do humans teach: On curriculum learning and teaching dimension.
\newblock In \emph{NIPS}, 2011.

\bibitem[Liu et~al.(2016)Liu, Zhu, and Ohannessian]{liu2016teaching}
Liu, Ji, Zhu, Xiaojin, and Ohannessian, H~Gorune.
\newblock The teaching dimension of linear learners.
\newblock In \emph{ICML}, 2016.

\bibitem[Meek et~al.(2016)Meek, Simard, and Zhu]{meek2016analysis}
Meek, Christopher, Simard, Patrice, and Zhu, Xiaojin.
\newblock Analysis of a design pattern for teaching with features and labels.
\newblock \emph{arXiv preprint arXiv:1611.05950}, 2016.

\bibitem[Mei \& Zhu(2015)Mei and Zhu]{mei2015using}
Mei, Shike and Zhu, Xiaojin.
\newblock Using machine teaching to identify optimal training-set attacks on
  machine learners.
\newblock In \emph{AAAI}, 2015.

\bibitem[Nemirovski et~al.(2009)Nemirovski, Juditsky, Lan, and
  Shapiro]{nemirovski2009robust}
Nemirovski, Arkadi, Juditsky, Anatoli, Lan, Guanghui, and Shapiro, Alexander.
\newblock Robust stochastic approximation approach to stochastic programming.
\newblock \emph{SIAM Journal on optimization}, 19\penalty0 (4):\penalty0
  1574--1609, 2009.

\bibitem[Pan \& Yang(2010)Pan and Yang]{pan2010survey}
Pan, Sinno~Jialin and Yang, Qiang.
\newblock A survey on transfer learning.
\newblock \emph{IEEE Transactions on knowledge and data engineering},
  22\penalty0 (10):\penalty0 1345--1359, 2010.

\bibitem[Pereira et~al.(2014)Pereira, Smith, and Yu]{pereira2014}
Pereira, Alfredo~F, Smith, Linda~B, and Yu, Chen.
\newblock {A Bottom-up View of Toddler Word Learning}.
\newblock \emph{Psychonomic bulletin \& review}, 21\penalty0 (1):\penalty0
  178--185, 2014.

\bibitem[Romero et~al.(2014)Romero, Ballas, Kahou, Chassang, Gatta, and
  Bengio]{romero2014fitnets}
Romero, Adriana, Ballas, Nicolas, Kahou, Samira~Ebrahimi, Chassang, Antoine,
  Gatta, Carlo, and Bengio, Yoshua.
\newblock Fitnets: Hints for thin deep nets.
\newblock \emph{arXiv preprint arXiv:1412.6550}, 2014.

\bibitem[Shinohara \& Miyano(1991)Shinohara and
  Miyano]{shinohara1991teachability}
Shinohara, Ayumi and Miyano, Satoru.
\newblock Teachability in computational learning.
\newblock \emph{New Generation Computing}, 8\penalty0 (4):\penalty0 337--347,
  1991.

\bibitem[Singla et~al.(2014)Singla, Bogunovic, Bartok, Karbasi, and
  Krause]{singla2014near}
Singla, Adish, Bogunovic, Ilija, Bartok, Gabor, Karbasi, Amin, and Krause,
  Andreas.
\newblock Near-optimally teaching the crowd to classify.
\newblock In \emph{ICML}, pp.\  154--162, 2014.

\bibitem[Yurovsky et~al.(2013)Yurovsky, Smith, and Yu]{yurovsky2013}
Yurovsky, Daniel, Smith, Linda~B, and Yu, Chen.
\newblock {Statistical Word Learning at Scale: The Baby's View is Better
  Developmental Science}.
\newblock \emph{Developmental Science}, 16\penalty0 (6):\penalty0 959--966,
  2013.

\bibitem[Zhu(2013)]{zhu2013machine}
Zhu, Xiaojin.
\newblock Machine teaching for bayesian learners in the exponential family.
\newblock In \emph{NIPS}, 2013.

\bibitem[Zhu(2015)]{zhu2015machine}
Zhu, Xiaojin.
\newblock Machine teaching: An inverse problem to machine learning and an
  approach toward optimal education.
\newblock In \emph{AAAI}, pp.\  4083--4087, 2015.

\end{thebibliography}

\bibliographystyle{icml2017}

\clearpage
\newpage

\appendix
\onecolumn

\begin{appendix}

\thispagestyle{plain}
\begin{center}
{\Large \bf Appendix}
\end{center}

\end{appendix}

\section{Details of the Proof}\label{appendix:proof}
\paragraph{Proof of Theorem~\ref{thm:fast_teaching}}
We assume the optimization starts with an initialized weights $w^0$. $t$ is denoted as the iteration index. Let $w_{g}^{t}$ and $w_{s}^{t}$ be the model parameter updated by our omniscient teacher and SGD, respectively. We first consider the case where $t=1$. For SGD, the first gradient update $w_{s}^{1}$ is
\begin{equation}
w_{s}^{1} = w^0 - \eta_t\, \frac{\partial \ell(\inner{w^0}{x_s},y_s)}{\partial w^0}.
\end{equation}
Then we compute the difference between $w_s^1$ and $w^*$:
\begin{equation}
\begin{aligned}
  \nbr{w_s^{1} - w^\ast}_2^2
  &= \nbr{w^0 - \eta_t\, \frac{\partial \ell(\inner{w^0}{x},y)}{\partial w^0} - w^\ast}_2^2 \\
  &= \nbr{w^0 - w^\ast}_2^2 + \eta_t^2 \nbr{\frac{\partial \ell(\inner{w^0}{x},y)}{\partial w^0}}_2^2 - 2\eta_t \inner{w^0 - w^\ast}{\frac{\partial \ell(\inner{w^0}{x},y)}{\partial w^0}}
\end{aligned}
\end{equation}
Because the omniscient teacher is to minimize last two term, so we are guaranteed to have
\begin{equation}
\nbr{w_g^{1} - w^\ast}_2^2\leq\nbr{w_s^{1} - w^\ast}_2^2.
\end{equation}
So with the same initialization $w_g^{0} = w_s^{0}$, $\nbr{w_g^{1} - w^\ast}_2^2\leq\nbr{w_s^{1} - w^\ast}_2^2$ is always true. Then we consider the case where $t=k, k\geq1$. We first compute the difference between $w_g^{k+1}$ and $w^*$:
\begin{equation}\label{wgk}
\begin{aligned}
  \nbr{w_g^{k+1} - w^\ast}_2^2
  &= \nbr{w_g^{k} - \eta_t\, \frac{\partial \ell(\inner{w_g^{k}}{x},y)}{\partial w^{k+1}} - w^\ast}_2^2 \\
  &= \nbr{w_g^{k} - w^\ast}_2^2 + \min_{\{x,y\}}\bigg{\{}\eta_t^2 \nbr{\frac{\partial \ell(\inner{w_g^{k}}{x},y)}{\partial w_g^{k}}}_2^2 - 2\eta_t \inner{w_g^{k} - w^\ast}{\frac{\partial \ell(\inner{w_g^{k}}{x},y)}{\partial w_g^{k}}}\bigg{\}}\\
  &= \nbr{w_g^{k} - w^\ast}_2^2 + \eta_t^2 \nbr{\frac{\partial \ell(\inner{w_g^{k}}{x^k_*},y^k_*)}{\partial w_g^{k}}}_2^2 - 2\eta_t \inner{w_g^{k} - w^\ast}{\frac{\partial \ell(\inner{w_g^{k}}{x^k_*},y^k_*)}{\partial w_g^{k}}}\\
  &=\nbr{w_g^{k} - w^\ast}_2^2 - TV(w_g^{k})
\end{aligned}
\end{equation}
where $x^k_*,y^k_*$ is the sample selected by the omniscient teacher in the $k$-th iteration. Using the given conditions, we can bound the difference between $w_s^{k+1}$ and $w^*$ from below:
\begin{equation}\label{wsk}
\begin{aligned}
  \nbr{w_s^{k+1} - w^\ast}_2^2
  &= \nbr{w_s^{k} - \eta_t\, \frac{\partial \ell(\inner{w_s^{k}}{x^s},y^s)}{\partial w_s^{k}} - w^\ast}_2^2 \\
  &= \nbr{w_s^{k} - w^\ast}_2^2 + \eta_t^2 \nbr{\frac{\partial \ell(\inner{w_s^{k}}{x^k_s},y^k_s)}{\partial w_s^{k}}}_2^2 - 2\eta_t \inner{w_s^{k} - w^\ast}{\frac{\partial \ell(\inner{w_s^{k}}{x^k_s},y^k_s)}{\partial w_s^{k}}}\\
  &\geq \nbr{w_s^{k} - w^\ast}_2^2- TV(w_s^{k})\\
\end{aligned}
\end{equation}
where $x^k_s,y^k_s$ is the sample selected by the random teacher in the $k$-th iteration. Comparing Eq. \ref{wgk} and Eq. \ref{wsk} and using the condition in the theorem, the following inequality always holds under the condition $\nbr{w_g^{k} - w^\ast}_2^2\leq \nbr{w_s^{k} - w^\ast}_2^2$:
\begin{equation}
\nbr{w_s^{k+1} - w^\ast}_2^2 = \nbr{w_s^{k} - w^\ast}_2^2- TV(w_s^{k}) \geq \nbr{w_g^{k} - w^\ast}_2^2 - TV(w_g^{k})= \nbr{w_g^{k+1} - w^\ast}_2^2.
\end{equation}
Further because we already know that $\nbr{w_g^{1} - w^\ast}_2^2\leq\nbr{w_s^{1} - w^\ast}_2^2$, using induction we can conclude that $\nbr{w_g^{t} - w^\ast}_2^2$ will be always not larger than $\nbr{w_s^{t} - w^\ast}_2^2$ ($t$ can be any iteration). Therefore, in each iteration the omniscient teacher can always converge not slower than random teacher (SGD).
\QED
\paragraph{Proof of Proposition~\ref{prop:square_loss}}
Consider the square loss $\ell(\inner{w}{x},y)=(\inner{w}{x}-y)^2$, we have $\frac{\partial\ell(\inner{w}{x},y)}{\partial w}=2(\inner{w}{x}-y)x$. Suppose we are given two initializations $w_1, w_2$ satisfying $\nbr{w_1-w_*}^2_2\leq\nbr{w_2-w_*}^2_2$. For square loss, we first write out
\begin{equation}
\begin{aligned}
&\nbr{w_1-w^*}^2-TV(w_1)= \nbr{w_1-w^*}^2 + \min_{x\in\mathcal{X},y\in\mathcal{Y}}\{\eta_t^2 T_1(x,y|w_1) - 2\eta_t T_2(x,y|w_1)\}\\
=&\nbr{w_1-w^*}^2+ \min_{\{x,y\}}\bigg{\{}\eta_t^2 \nbr{\frac{\partial \ell(\inner{w_1}{x},y)}{\partial w_1}}_2^2 - 2\eta_t \inner{w_1 - w^\ast}{\frac{\partial \ell(\inner{w_1}{x^*},y^*)}{\partial w_1}}\bigg{\}}\\
=&\nbr{w_1-w^*}^2+\left\{ {\begin{array}{*{20}{l}}
{2(\frac{R}{\nbr{w_1-w^*}})^2\nbr{w_1-w^*}^2(w_1-w^*)},~~\textnormal{if}~ \frac{R}{\nbr{w_1-w^*}}<\frac{1}{\eta_t}\\
{-\nbr{w_1-w^*}^2,~~\textnormal{if}~ \frac{R}{\nbr{w_1-w^*}}\geq\frac{1}{\eta_t}}
\end{array}} \right.
\end{aligned}
\end{equation}
Similarly for $w_2$, we have
\begin{equation}
\begin{aligned}
&\nbr{w_2-w^*}^2-TV(w_2)\\
=&\nbr{w_2-w^*}^2+\left\{ {\begin{array}{*{20}{l}}
{2(\frac{R}{\nbr{w_2-w^*}})^2\nbr{w_2-w^*}^2(w_2-w^*)},~~\textnormal{if}~ \frac{R}{\nbr{w_2-w^*}}<\frac{1}{\eta_t}\\
{-\nbr{w_2-w^*}^2,~~\textnormal{if}~ \frac{R}{\nbr{w_2-w^*}}\geq\frac{1}{\eta_t}}
\end{array}} \right.
\end{aligned}
\end{equation}
There will be three scenarios to consider: (1) $R\eta_t\leq\nbr{w_1-w^*}\leq\nbr{w_2-w^*}$; (2) $\nbr{w_1-w^*}\leq R\eta_t\leq\nbr{w_2-w^*}$; (3) $\nbr{w_1-w^*}\leq\nbr{w_2-w^*}\leq R\eta_t$. It is easy to verify that under all three scenarios, we have
\begin{equation}
\begin{aligned}
&\nbr{w_1-w^*}^2-TV(w_1)\leq&\nbr{w_2-w^*}^2-TV(w_2)
\end{aligned}
\end{equation}
\QED
To simplify notations, we denote $\beta_{\rbr{\inner{w}{x}, y}}=\nabla_{\inner{w}{x}}\ell\rbr{\inner{w}{x}, y}$ for a loss function $\ell(\cdot, \cdot)$ in the following proof. For omniscient teacher, $(\hat x,\hat y)$ denotes a specific construction of $(x,y)$. Notice that $(\xtil,\ytil)$ will not be used in omniscient teacher case to avoid ambiguity, since the student and the teacher use the same representation space.

\paragraph{Proof of Theorem~\ref{thm:opt_synthesis}}
At $t$-step, the omniscient teacher selects the samples via optimization
$$
\min_{x\in \Xcal, y\in \Ycal}\eta^2 \|\nabla_{w^t} \ell\rbr{\inner{w^t}{x}, y}\|^2 - 2\eta\inner{w^t - w^\ast}{\nabla_{w^t} \ell\rbr{\inner{w^t}{x}, y}}.
$$
We denote $\hat x = \gamma\rbr{w^t-w^\ast}$ and $\hat y \in \Ycal$, since $\gamma \rbr{w-w^\ast}\in \Xcal$, we have
\begin{eqnarray}\label{eq:reduction}
&&\min_{x\in \Xcal, y\in \Ycal}\eta^2 \|\nabla_{w^t} \ell\rbr{\inner{w^t}{x}, y}\|^2 - 2\eta\inner{w^t - w^\ast}{\nabla_{w^t} \ell\rbr{\inner{w^t}{x}, y}}\\
&\le& \rbr{\eta^2\beta_{(\inner{w^t}{\hat x}, \hat y)}^2\gamma^2 - 2\eta\beta_{(\inner{w^t}{\hat x}, \hat y)}\gamma}\|w^t - w^\ast\|_2^2.
\end{eqnarray}
Plug Eq.~\eq{eq:reduction} into the recursion Eq.~\eq{graddecomp}, we have
\begin{equation}\label{thm4_recursion}
\begin{aligned}
  \nbr{w^{t+1} - w^\ast}_2^2
  &= \min_{x\in \Xcal, y\in \Ycal}\nbr{w^t - \eta\, \frac{\partial \ell(\inner{w}{x},y)}{\partial w} - w^\ast}_2^2 \\
  &= \nbr{w^t - w^\ast}_2^2 + \min_{x\in \Xcal, y\in \Ycal}\eta^2{\nbr{\frac{\partial \ell(\inner{w^t}{x},y)}{\partial w^t}}_2^2}- 2\eta {\inner{w^t - w^\ast}{\frac{\partial \ell(\inner{w^t}{x},y)}{\partial w^t}}} \\
  &\le \rbr{1 + \eta^2\beta_{(\inner{w^t}{\hat x}, \hat y)}^2\gamma^2 - 2\eta\beta_{(\inner{w^t}{\hat x}, \hat y)}\gamma}\|w^t - w^\ast\|_2^2 = \rbr{1 - \eta\beta_{(\inner{w^t}{\gamma\rbr{w^t - w^\ast}}, \hat y)}\gamma}^2\|w^t - w^\ast\|_2^2.
\end{aligned}
\end{equation}
First we let $\nu(\gamma)=\min_{w,y}\gamma\nabla_{\inner{w}{\gamma\rbr{w-w^\ast}}}\ell\rbr{\inner{w}{\gamma\rbr{w-w^\ast}}, y}$. Then we have the condition $0< \nu(\gamma) \le \gamma\beta_{(\inner{w}{\gamma\rbr{w-w^\ast}}, \hat y)}\le \frac{1}{\eta}<\infty$ for any $w,y$, so we can obtain
$$
0\le 1 - \gamma\eta\beta_{(\inner{w}{\gamma\rbr{w-w^\ast}}, \hat y)} \le 1 - \eta\nu(\gamma),
$$
after simplifying $\nu(\gamma)$ to $\nu$, we therefore have the following inequality from Eq. \eqref{thm4_recursion}:
$$
\nbr{w^{t+1} - w^\ast}_2^2 \le \rbr{1 - \eta\nu}^2 \nbr{w^{t} - w^\ast}_2^2,
$$
Thus we can have the exponential convergence:
$$
\nbr{w^{t} - w^\ast}_2 \le \rbr{1 - \eta\nu}^{t}\nbr{w^{0} - w^\ast}_2,
$$
in other words, the student needs $\rbr{\log\frac{1}{1 - \eta\nu}}^{-1}\log\frac{\|w^0 - w^\ast\|}{\epsilon}$ samples to achieve an $\epsilon$-approximation of $w^\ast$.

\QED

\paragraph{Proof of Proposition~\ref{prop:general_teachable}}

Because $\ell\rbr{\inner{w}{x}, y}$ is $\zeta_1$-strongly convex w.r.t. $w$, we have
$$
\zeta_1\rbr{\ell\rbr{\inner{w}{x}, y} - \min_{w}\ell\rbr{\inner{w}{x}, y}}\le \nbr{\nabla_w \ell\rbr{\inner{w}{x}, y}}^2 = \beta_{(\inner{w}{x}, y)}^2\nbr{x}^2, \quad \forall \cbr{x, y}\in \Xcal\times\Ycal,
$$
where $\Xcal = \cbr{x\in \RR^d, \nbr{x}\le R}$. Using $\hat{x}=\gamma(w-w^*), \gamma\geq0$, we have
$$
\sqrt{\zeta_1\rbr{\ell\rbr{\inner{w}{\gamma(w-w^*)}, y} - \min_{w}\ell\rbr{\inner{w}{\gamma(w-w^*)}, y}}} \le \beta_{(\inner{w}{\gamma(w-w^*)}, y)}\gamma\|w-w^*\|.
$$
We assume the loss function is always non-negative, i.e., $\ell\rbr{\inner{w}{x}, y}\geq 0$. Therefore we have
$$
\sqrt{\zeta_1\rbr{\ell\rbr{\inner{w}{\gamma(w-w^*)}, y}}} \le \beta_{(\inner{w}{\gamma(w-w^*)}, y)}\gamma\|w-w^*\|.
$$
Because $\ell\rbr{\inner{w}{x}, y}$ is $\zeta$-strongly convex w.r.t. $w$, it is also $\zeta_2$-strongly convex w.r.t. $\inner{w}{x}$. Then we perform Taylor expansion to $\ell\rbr{\inner{w}{\gamma(w-w^*)}, y}$ w.r.t. $\inner{w}{x}$ at the point $\inner{w^*}{x}$ and obtain
$$
\ell\rbr{\inner{w}{\gamma(w-w^*)}, y} \geq \ell\rbr{\inner{w}{\gamma(w^*-w^*)}, y} + \nabla_{\inner{w}{x}}\ell\rbr{\inner{w}{\gamma(w^*-w^*)}, y}(w-w^*)^Tx + \frac{\zeta_2}{2}\|(w-w^*)^Tx\|^2
$$
which leads to
$$
\ell\rbr{\inner{w}{\gamma(w-w^*)}, y} \geq \frac{\zeta_2}{2}\gamma^2\|w-w^*\|^4
$$
Combining pieces, we have
$$
\sqrt{\frac{\zeta_1\zeta_2}{2}}\gamma\|w-w^*\| \le \beta_{(\inner{w}{\gamma(w-w^*)}, y)}\gamma.
$$
Then if we set $\gamma=\min\big{\{} \sqrt{\frac{2}{\zeta_1\zeta_2}}\frac{1}{\|w-w^*\|\eta},  \frac{R}{\|w-w^*\|} \big{\}}$, we can have $\frac{1}{\eta}\leq\beta_{(\inner{w}{\gamma(w-w^*)}, y)}\gamma$.
Because $\ell\rbr{\inner{w}{x}, y}$ is Lipschitz smooth w.r.t. $\inner{w}{x}$ with parameter $L$, we have
$$
\nbr{\beta_{(\inner{w}{ x}, y)}-\beta_{(\inner{w^*}{ x}, y)}}\leq LR\nbr{w-w^*}
$$
Because $\beta_{(\inner{w^*}{ x}, y)}=0$, we have the following inequality:
$$
\nbr{\beta_{(\inner{w}{ x}, y)}}\leq LR\nbr{w-w^*}
$$
If we multiply both side with $\gamma$, we can have
$$
\beta_{(\inner{w}{ x}, y)}\gamma\leq LR\nbr{w-w^*}\gamma
$$
By setting $\gamma$ as $\frac{1}{LR\eta\nbr{w-w^*}}$, we arrive at $\beta_{(\inner{w}{ x}, y)}\gamma<\frac{1}{\eta}$. Combining pieces, as long as we set
$$
\gamma=\min\bigg{\{} \sqrt{\frac{2}{\zeta_1\zeta_2}}\frac{1}{\eta\|w-w^*\|},  \frac{R}{\|w-w^*\|},  \frac{1}{LR\eta\nbr{w-w^*}} \bigg{\}},
$$
then we can have
$$
0<c\leq\beta_{(\inner{w}{\gamma \hat x}, \hat y)}\gamma \le \frac{1}{\eta}.
$$
where $c$ is a non-zero positive constant. Therefore, we achieve the condition for the exponential synthesis-based teaching.

\QED

By the Proposition~\ref{prop:general_teachable}, the absolute loss and sqaure loss are exponentially teachable in synthesis-based case, and we can obtain $\gamma$ by plugging into the general form. We will tighten the $\gamma$ up by analyzing absolute loss and square loss separately. Besides that, we also show the commonly used loss functions for classification, \eg, hinge loss and logistic loss, are also exponentially teachable in synthesis-based teaching if $\nbr{w^\ast}$ can be bounded.

\begin{proposition}
Absolute loss is exponentially teachable in synthesis-based teaching.
\end{proposition}

\begin{proof}
To show one loss function is exponentially teachable in synthesis-based case, we just need to find the appropriate $\gamma$ such that the learning intensity is bounded below and above, according to Theorem \ref{thm:opt_synthesis}. For the absolute loss, \ie,
$$
\ell\rbr{\inner{w}{x}, y} = \abr{\inner{w}{x} - y},
$$
its sub-gradient is
$$
\nabla_{w}\ell(\inner{w}{x}, y) = \sgn(\inner{w}{x} - y)x,
$$
and thus, the learning intensity $\beta_{(\inner{w}{x}, y)} = \sgn\rbr{\inner{w}{x} - y}$. For $w\neq w^\ast$, plugging $\hat x= \gamma\rbr{w - w^\ast}$ and $\hat y = \inner{w^\ast}{\gamma\rbr{w - w^\ast}}$ into the learning intensity, we have
$$
\beta_{\gamma\inner{w}{\hat x}, \hat y}\gamma = \sgn\rbr{\gamma^2\inner{w - w^\ast}{w - w^\ast}}\gamma = \gamma.
$$
Recall that $\gamma\neq 0$, $\abr{\gamma} \le \frac{R}{\nbr{w^t-w^\ast}}$, $\forall t\in \NN$, we have
$$
\gamma \le \min_{t\in \NN} \frac{R}{\nbr{w^t-w^\ast}} :=C.
$$
Set $\gamma = \min\{C, \frac{1}{\eta}\}$, we have $\nu = \min\{C, \frac{1}{\eta}\}$. Therefore, we obtain the exponential decay. In fact, since the $\nbr{w^t-w^\ast}$ decreases in every step, we have $C = \frac{R}{\nbr{w^0-w^\ast}}$. In following proof, we will follow the same argument to use this fact.
\end{proof}

\begin{proposition}
Square loss is exponentially teachable in synthesis-based teaching.
\end{proposition}

\begin{proof}
For square loss, \ie,
$$
\ell\rbr{\inner{w}{x}, y} = \rbr{\inner{w}{x} - y}^2,
$$
its gradient is
$$
\nabla_{w}\ell\rbr{\inner{w}{x}, y} = 2\rbr{\inner{w}{x} - y}x,
$$
and thus, the learning intensity $\beta_{\inner{w}{x}, y} = 2\rbr{\inner{w}{x} - y}$. For $w\neq w^\ast$, plugging $\hat x = \gamma\rbr{w - w^\ast}$ and $\hat y = \inner{w^\ast}{\gamma\rbr{w - w^\ast}}$ into the learning intensity, we have
$$
\beta_{(\inner{w}{\hat x}, \hat y)} \gamma = 2\gamma^2\nbr{w - w^\ast}^2.
$$
Set $\gamma = \min\cbr{\frac{1}{\sqrt{2\eta}\nbr{w^t - w^\ast}}, \frac{R}{\nbr{w^t - w^\ast}}}$, we achieve the exponential teachable condition.
\end{proof}

\begin{proposition}
Hinge loss is exponentially teachable in synthesis-based teaching if $\nbr{w^\ast}\le 1$.
\end{proposition}

\begin{proof}
For hinge loss, \ie,
$$
\ell\rbr{\inner{w}{x}, y} = \max\rbr{1 - y\inner{w}{x}, 0},
$$
as long as $1 - y\inner{w}{x} >0$, its subgradient will be
$$
\nabla_{w}\ell\rbr{\inner{w}{x}, y} = -yx.
$$
Denote $\hat x = \gamma\rbr{w - w^\ast}$, we have $\beta_{\inner{w}{\hat x}, \hat y} = -\hat y$ where $\hat y \in \cbr{-1, 1}$. To satisfy the exponential teachable condition, we need to select $\hat y$ and $\gamma$ such that
\begin{eqnarray*}
  \begin{cases}
    1 - \hat y \inner{w}{\hat x} >0   \\
    0<-\hat y\gamma \le \frac{1}{\eta}\\
    \abr{\gamma} \le \frac{R}{\nbr{w-w^\ast}}\\
  \end{cases}\Rightarrow
  \begin{cases}
    \hat y\gamma \inner{w}{w - w^\ast} <1   \\
    -\frac{1}{\eta}\le \hat y\gamma <0\\
    \abr{\gamma} \le \frac{R}{\nbr{w-w^\ast}}\\
  \end{cases}\Rightarrow
  \begin{cases}
    \inner{w}{w - w^\ast} >-1   \\
    -\frac{1}{\eta}\le \hat y\gamma <0\\
    \abr{\gamma} \le \frac{R}{\nbr{w-w^\ast}}\\
  \end{cases}.
\end{eqnarray*}
If $\nbr{w^\ast}\le 1$, we can show
\begin{eqnarray*}
\inner{w}{w^\ast}\le \nbr{w}\nbr{w^\ast} \le \nbr{w} < 1 + \nbr{w}^2,
\end{eqnarray*}
where the last inequality comes from the fact $1 + a^2 -a >0$, and thus, we have $\inner{w}{w - w^\ast} >-1$. Therefore, we select any configuration of $\hat y$ and $\gamma$ satisfying
\begin{eqnarray*}
   -\frac{1}{\eta}\le \hat y\gamma <0,\quad\text{and}\quad\abr{\gamma} \le \frac{R}{\nbr{w-w^\ast}}.
\end{eqnarray*}
Particularly, we set $\hat y = -1$ and $\gamma = \min\cbr{\frac{1}{\eta}, \frac{R}{\nbr{w^0-w^\ast}}}$.

\end{proof}

\begin{proposition}
Logistic loss is exponentially teachable in synthesis-based teaching if $\nbr{w^\ast}\le 1$.
\end{proposition}

\begin{proof} For the logistic loss, \ie,
$$
\ell\rbr{\inner{w}{x}, y} = \log\rbr{1 + \exp(-y\inner{w}{x})},
$$
its gradient is
$$
\nabla_w\ell\rbr{\inner{w}{x}, y} = -\frac{yx}{1 + \exp(y\inner{w}{x})}.
$$
Denote $\hat x = \gamma\rbr{w - w^\ast}$, we have $\beta_{\inner{w}{\hat x}, \hat y} = -\frac{\hat y}{1 + \exp(\hat y\inner{w}{\hat x})}$ where $\hat y \in \cbr{-1, 1}$. To satisfy the exponential teachable condition, we need to select $\hat y$ and $\gamma$ such that
\begin{eqnarray*}
  \begin{cases}
    0<-\frac{\hat y\gamma}{1 + \exp(\hat y\inner{w}{\hat x})}\le \frac{1}{\eta}\\
    \abr{\gamma} \le \frac{R}{\nbr{w-w^\ast}}\\
  \end{cases}.
\end{eqnarray*}
Particularly, we set $\hat y = -1$, we can fix the $\gamma$ by
\begin{eqnarray*}
    0< \frac{\gamma}{1 + \exp(\gamma)}<\frac{\gamma}{1 + \exp(\hat y\inner{w}{\hat x})}\le \gamma\le \frac{1}{\eta},\quad\text{and}\quad\abr{\gamma} \le \frac{R}{\nbr{w-w^\ast}}.
\end{eqnarray*}
The $\frac{\gamma}{1 + \exp(\gamma)}<\frac{\gamma}{1 + \exp(\hat y\inner{w}{\hat x})}$ is obtained by the monotonicity of $\exp(\cdot)$ and $\inner{w}{w - w^\ast} >-1$ when $\nbr{w^\ast}$. Therefore, we can choose $\gamma = \min\cbr{\frac{1}{\eta}, \frac{R}{\nbr{w^0-w^\ast}} }$, and thus, the lower bound $\nu =  \frac{\gamma}{1 + \exp(\gamma)}$.
\end{proof}

\paragraph{Proof of Corollary~\ref{cor:opt_combination}}
In each update, given the training sample $x\in\text{span}\rbr{\Xcal}$, we have $w^{t+1} = w^{t} - \eta\beta_{\inner{w}{x}, y}x$, therefore, the $\Delta_{t+1} w := w^{t+1} - w^0\in \text{span}\rbr{\Xcal}$. If $w^0-w^\ast \in \text{span}\rbr{\Xcal}$, $w^{t+1} - w^\ast \in \text{span}\rbr{\Xcal}$, which means by linear combination, we can construct $\hat\gamma\sum_{i=1}^n\alpha_i^t x_i = \gamma \rbr{w^t - w^\ast}$. With the condition that the loss function is exponentially synthesis-based teachable, we achieve the conclusion that the combination-based omniscient teacher will converge at least exponentially with the same rate to the synthesis-based teaching.
\QED

\paragraph{Proof of Theorem~\ref{thm:opt_pool}}
The proof is similar to the synthesis-based case. However, we introduce the consideration of the effect of pool-based teaching. Specifically, we first obtain a virtual training sample in full space, and then, we generate the sample from the candidate pool to mimic the virtual sample.

With the condition $w^0 - v^\ast\in \text{span}\rbr{\Dcal}$, as we discussed in the proof of Corollary~\ref{cor:opt_combination}, in every iteration, $w^t - v^\ast\in \text{span}\rbr{\Dcal}$. Therefore, we only need to consider in the space of $\text{span}\rbr{\Dcal}$. Meanwhile, since the teacher can rescale the sample, without loss of generality, we assume if $x\in \Xcal$, then $-x\in\Xcal$ to make the rescaling is always positive.

At $t$-step, as the loss is exponentially synthesis-based teachable with $\gamma$, therefore, we have the virtually constructed sample $\cbr{x_v, y_v}$ where $x_v = \gamma \rbr{w^t - w^\ast}$ with $\gamma$ satisfying the condition of exponentially teachable in synthesis-based settings, we first rescale the candidate pool $\Xcal$ such that
$$
\forall x\in \Xcal, \gamma_x\nbr{x} =\nbr{x_v} = \gamma\nbr{w^t - w^\ast}.
$$
We denote the rescaled candidate pool as $\Xcal_t$, under the condition of rescalable pool-based teachability, there is a sample $\cbr{\hat x, \hat y}\in \Xcal \times \Ycal$ with scale factor $\hat \gamma$ such that
\begin{eqnarray*}
\min_{\rbr{x, y}\in \Xcal_t\times \Ycal} &&\eta^2 \|\nabla_{w^t} \ell\rbr{\inner{w^t}{x}, y}\|^2 - 2\eta\inner{w^t - w^\ast}{\nabla_{w^t} \ell\rbr{\inner{w^t}{x}, y}}\\
&\le& \eta^2 \beta_{\inner{w^t}{\hat\gamma \hat x}, \hat y}^2\nbr{\hat x}^2 - 2\eta \beta_{\inner{w^t}{\hat\gamma \hat x}, \hat y}\langle w^t - w^\ast, \hat\gamma \hat x \rangle.
\end{eqnarray*}
We decompose the $\hat\gamma \hat x = a x_v + {x_v}_\perp$ with $a = \frac{\inner{\hat\gamma \hat x}{{x_v}}}{\nbr{{x_v}}^2}$. and ${x_v}_\perp = \hat\gamma \hat x - a{x_v}$. Then, we have
\begin{eqnarray*}
\min_{\rbr{x, y}\in \Xcal_t\times \Ycal} &&\eta^2 \|\nabla_{w^t} \ell\rbr{\inner{w^t}{x}, y}\|^2 - 2\eta\inner{w^t - w^\ast}{\nabla_{w^t} \ell\rbr{\inner{w^t}{x}, y}}\\
&\le& \eta^2 \beta_{\inner{w^t}{\hat\gamma \hat x}, \hat y}^2\nbr{\hat x}^2 - 2\eta \beta_{\inner{w^t}{\hat\gamma \hat x}, \hat y}\langle w^t - w^\ast, \hat\gamma \hat x \rangle\\
&=&\eta^2 \beta_{\inner{w^t}{\hat\gamma \hat x}, \hat y}^2\gamma^2\nbr{w - w^\ast}^2 - 2\eta \beta_{\inner{w^t}{\hat\gamma \hat x}, \hat y}\langle w^t - w^\ast, a {x_v} + {x_v}_\perp\rangle\\
&=&\eta^2 \beta_{\inner{w^t}{\hat\gamma \hat x}, \hat y}^2\gamma^2\nbr{w - w^\ast}^2 - 2\eta \beta_{\inner{w^t}{\hat\gamma \hat x}, \hat y}\gamma a\nbr{w^t - w^\ast}^2.
\end{eqnarray*}
Under the condition
$$
0< {\gamma\beta_{\inner{w}{\gamma \frac{w - w^\ast}{\hat x}}, \hat y}}< \frac{{2\Vcal(\Xcal)}}{\eta},
$$
we denote $\nu\rbr{\gamma} = \min_{w, \hat x\in\Xcal, \hat y\in \Ycal} \gamma\beta_{\inner{w}{\gamma \frac{w - w^\ast}{\hat x}}, \hat y}>0$ and $\mu\rbr{\gamma} = \max_{w, \hat x\in \Xcal, \hat y\in \Ycal}\gamma\beta_{\inner{w}{\gamma \frac{w - w^\ast}{\hat x}}, \hat y}<\frac{2\Vcal(\Xcal)}{\eta}$.

we have the recursion
$$
\nbr{w^{t+1} - w^\ast}_2^2 \le r(\eta, \gamma)\nbr{w^{t} - w^\ast}_2^2,
$$
with $r(\eta, \gamma, \Vcal(\Xcal)) := {\max\cbr{1 + \eta^2\mu\rbr{\gamma}^2 - 2\eta\mu\rbr{\gamma} \Vcal(\Xcal), 1 + \eta^2\nu\rbr{\gamma}^2 - 2\eta\nu\rbr{\gamma} \Vcal(\Xcal)}}$ and $0\le r(\eta, \gamma) <1$. Therefore, the algorithm converges exponentially
$$
\nbr{w^{t} - w^\ast}_2 \le {r\rbr{\eta, \gamma}}^{t/2}\nbr{w^{0} - w^\ast}_2,
$$
in other words, the student needs $2\rbr{\log\frac{1}{r\rbr{\eta, \gamma, \Vcal(\Xcal)}}}^{-1}\log\frac{\|w^0 - w^\ast\|}{\epsilon}$ samples to achieve an $\epsilon$-approximation of $w^\ast$. For clearity, we define the constant term as $C_2^{\eta, \gamma, \Vcal(\Xcal)}=2\rbr{\log\frac{1}{r\rbr{\eta, \gamma, \Vcal(\Xcal)}}}^{-1}$.
\QED

\clearpage
\newpage

\section{Detailed Experimental Setting}\label{appendix:exp}

\begin{table*}[h]
\centering
\footnotesize
\begin{tabular}{|c|c|c|c|c|c|}
\hline
Layer & CNN-6 & CNN-9 & CNN-12 \\
\hline\hline
Conv1.x  & [3$\times$3, 16]$\times$2 & [3$\times$3, 16]$\times$3 & [3$\times$3, 16]$\times$4  \\\hline
Pool1&\multicolumn{3}{c|}{2$\times$2 Max, Stride 2}\\\hline
Conv2.x  & [3$\times$3, 32]$\times$2 & [3$\times$3, 32]$\times$3 & [3$\times$3, 32]$\times$4 \\\hline
Pool2 & \multicolumn{3}{c|}{2$\times$2 Max, Stride 2}\\\hline
Conv3.x & [3$\times$3, 64]$\times$2  & [3$\times$3, 64]$\times$3 & [3$\times$3, 64]$\times$4 \\\hline
Pool3 & \multicolumn{3}{c|}{2$\times$2 Max, Stride 2}\\\hline
FC1 & 32 & 32 & 32 \\\hline
\end{tabular}
\caption{Our standard CNN architectures for CIFAR-10. Conv1.x, Conv2.x and Conv3.x denote convolution units that may contain multiple convolution layers. E.g., [3$\times$3, 16]$\times$3 denotes 3 cascaded convolution layers with 16 filters of size 3$\times$3. The CNNs learning ends at 20K iterations with multi-step rate decay.}\label{netarch}
\end{table*}

\paragraph{General Settings}
We have used three linear models in the experiments. In specific, the formulation of ridge regression (RR) is
$$
\min_{w\in\mathbb{R}^d,b\in\mathbb{R}}\frac{1}{n}\sum_{i=1}^{n}\frac{1}{2}(w^Tx_i+b-y_i)^2+\frac{\lambda}{2}\nbr{w}^2
$$
The formulation of logistic regression (LR) is
$$
\min_{w\in\mathbb{R}^d,b\in\mathbb{R}}\frac{1}{n}\sum_{i=1}^{n}\log(1+\exp\{-y_i(w^Tx_i+b)\})+\frac{\lambda}{2}\nbr{w}^2
$$
The formulation of support vector machine (SVM) is
$$
\min_{w\in\mathbb{R}^d,b\in\mathbb{R}}\frac{1}{n}\sum_{i=1}^{n}\max(1-y_i(w^Tx_i+b),0)+\frac{\lambda}{2}\nbr{w}^2
$$

\paragraph{Comparison of different teaching strategies}
We use a linear regression model (ridge regression with $\lambda=0$) for this experiment. We set $R$ as 1 and uniformly generate 30 data points as our knowledge pool for the teacher. In this first case, we set the feature dimension as 2, while in the second case, feature dimension is 70. The learning rate is set as 0.0001 for pool-based teaching, same as BGD and SGD.
\paragraph{Experiments on Gaussian data}
Specifically, RR is run on training data $(\bm{x}_i,y)$ where each entry in $\bm{x}_i$ is Gaussian distributed and $y = \inner{\bm{w}^*}{\bm{x}_i}+\epsilon$. LR and SVM are run on $\{\mathcal{X}_1,+1\}$ and $\{\mathcal{X}_2,-1\}$ where $\bm{x}_i\in\mathcal{X}_1$ is Gaussian distributed in each entry and $+1,-1$ are the labels. Specifically, we use the 10-dimension data that is Gaussian distributed with $(0.5,\cdots,0.5)$ (label $+1$) and $(-0.5,\cdots,-0.5)$ (label $-1$) as mean and identity matrix as covariance matrix. We generate 1000 training data points for each class. Learning rate for the same feature space is 0.0001, while learning rate for different feature spaces are 0.00001. $\lambda$ is set as 0.00005.
\paragraph{Experiments on uniform spherical data} We first generate the training data that are uniformly distributed on a unit sphere $\thickmuskip=2mu \medmuskip=2mu\|\bm{x}_i\|_2=1$. Then we set the data points on half of the sphere ($(0,\pi]$) as label $+1$ and the other half ($(\pi,2\pi]$) as label $-1$. All the generated data points are 2D. For the scenario of different features, we use a random orthogonal projection matrix to generate the teacher's feature space from student's. Learning rate for the same feature space is 0.001, while learning rate for different feature spaces are 0.0001. $\lambda$ is set as 0.00005.
\paragraph{Experiments on MNIST dataset}
We use 24D random features (projected by a random matrix $\mathbb{R}^{784\times 24}$) for the MNIST dataset. The learning rate for all the compared methods are 0.001. Note that, we generate the teacher's features using a random projection matrix ($\mathbb{R}^{24\times 24}$) from the original 24D student's features. $\lambda$ is set as 0.00005.

\paragraph{Experiments on CIFAR-10 dataset} The learning rate for all the compared methods are 0.001. $\lambda$ is set as 0.00005. The goal is to learn the $\mathbb{R}^{32\times 10}$ fully connected layer, which is also the classifiers for 10 classes. The three network we use in the experiments are shown as follows:

\paragraph{Experiments on infant ego-centric dataset} We manually crop and label all the objects that the child is holding for this experiments. For feature extraction, we use VGG-16 network that is pre-trained on Imagenet dataset. Then we use PCA to reduce the 4096 dimension to 64 dimension. We train a multi-class logistic regression to classify the objects. Note that, the omniscient teacher is also applied to train the logistic regression model. The learning rate is set to 0.001 for both SGD and omniscient teacher.

\clearpage
\newpage

\section{Comparison of different teaching strategies}
\label{appendix:diffstra}

\begin{figure}[h]
  \centering
  \footnotesize
  \includegraphics[width=4in]{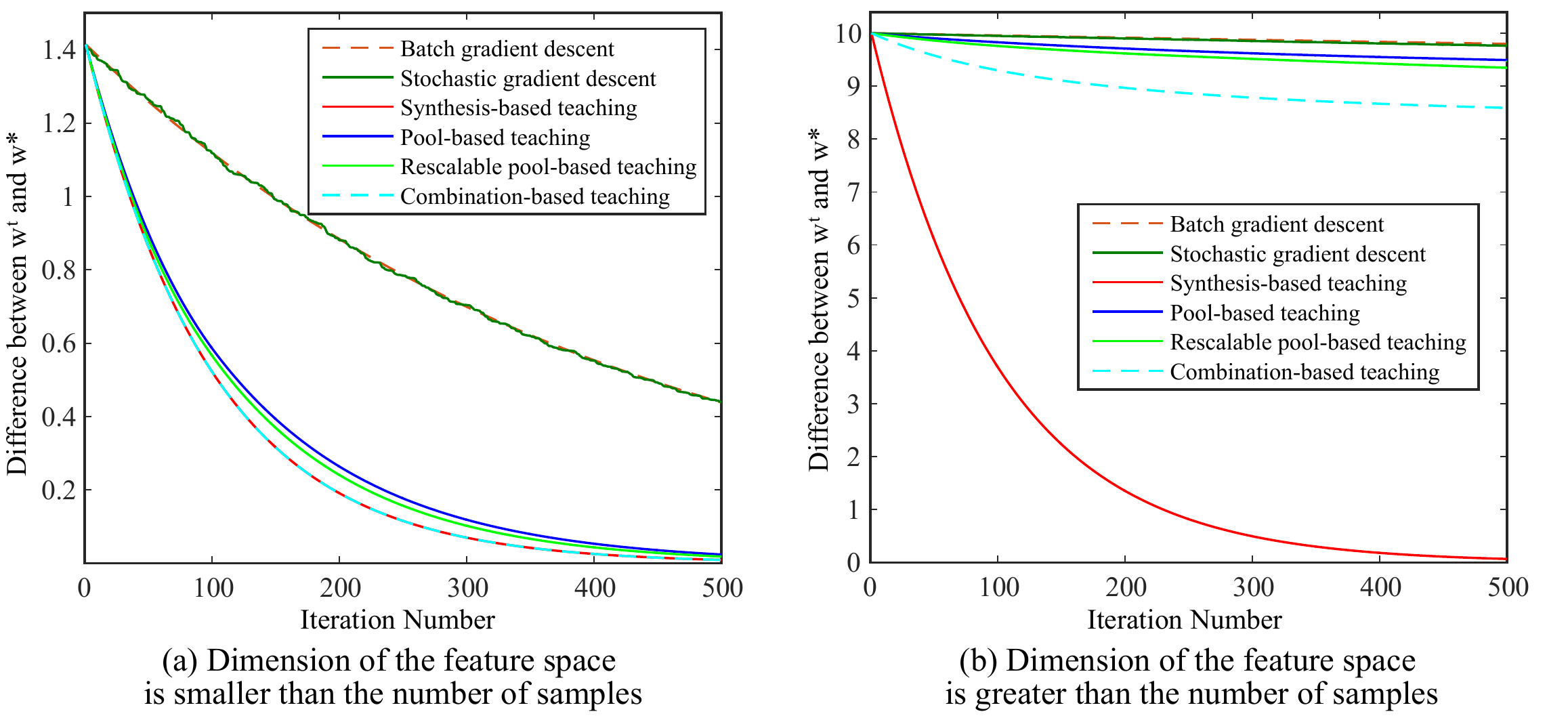}
  \vspace{-3mm}
  \caption{Comparison of different teaching strategies.}\label{synt}
\end{figure}
We first compare four different teaching strategies for the omniscient teacher. We consider two scenarios. One is that the dimension of feature space is smaller than the number of samples (the given features are sufficient to represent the entire feature), and the other is that the feature dimension is greater than the number of samples (the given features are not sufficient to represent the entire feature). In these two scenarios, we find that synthesis-based teaching usually works the best and always achieves exponential convergence. The combination-based teaching is exactly the same as the synthesis-based teaching in the first scenario, but it is much worse than synthesis in the second scenario. Rescalable pool-based teaching is also better than pool-based teaching. Empirically, the experiment verifies our theoretical findings: the more flexible the teaching strategy is, the more convergence gain we may obtain.

\clearpage
\newpage

\section{More experiments on MNIST dataset}
\label{appendix:mnist}
We provide more experimental results on MNIST dataset. Fig. \ref{mnist_vis_79} shows the selected examples from 7/9 binary digit classification. The results further verify the teacher models tend to select easy examples at first and gradually shift their focuses to difficult examples, very much resembling the human learning. Fig. \ref{mnist_diffw} shows the difference between the current model parameter and the optimal model parameter over iterations. It also shows that our teachers achieve faster convergence.

\begin{figure}[h]
  \centering
  \footnotesize
  \includegraphics[width=4.3in]{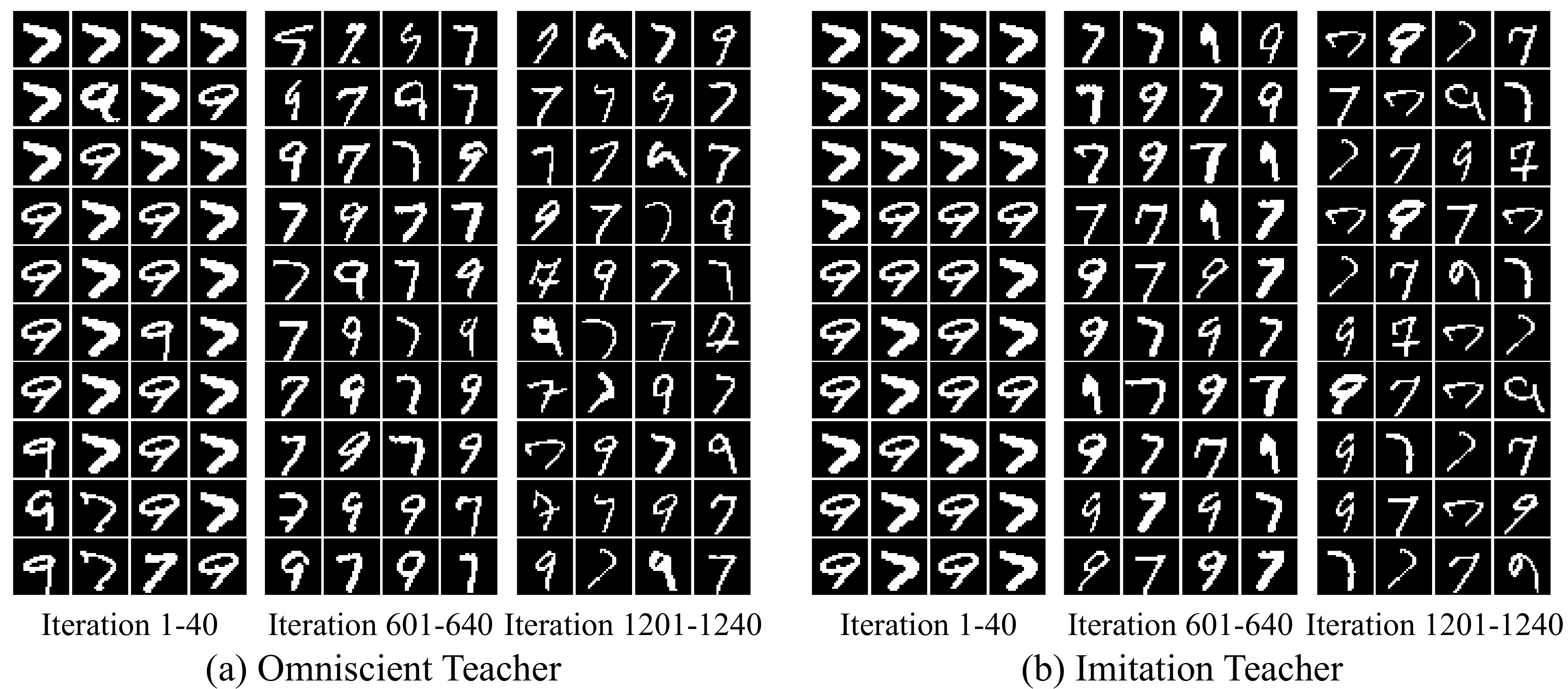}
  \vspace{-1mm}
  \caption{Selected training examples during iteration. (7/9 classification)}\label{mnist_vis_79}
\end{figure}

\begin{figure}[h]
  \centering
  \footnotesize
  \includegraphics[width=4in]{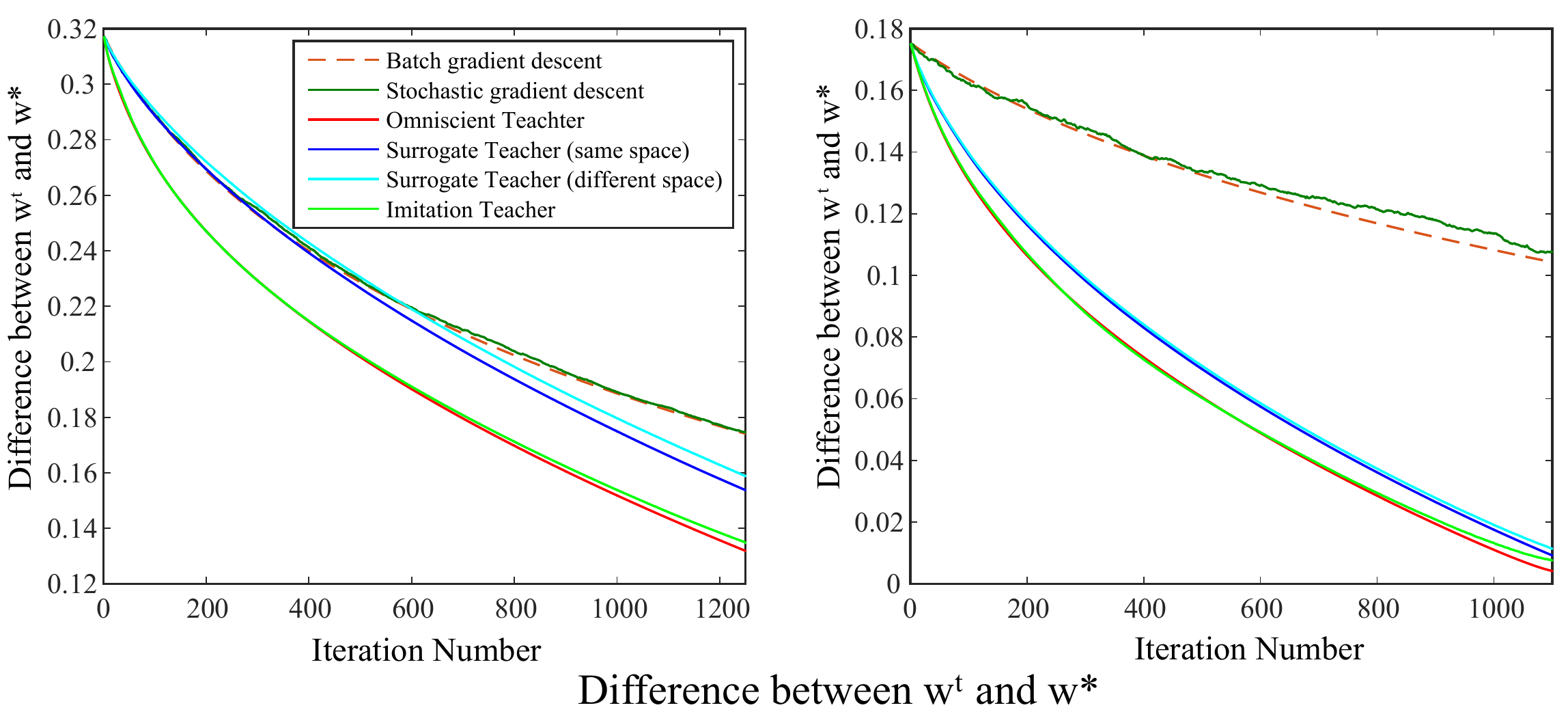}
  \vspace{-1mm}
  \caption{Teaching logistic regression on MNIST dataset. Left column: 0/1 classification. Right column: 3/5 classification}\label{mnist_diffw}
\end{figure}

\clearpage
\newpage

\section{Teaching linear models on uniform spherical data}
\label{appendix:spherical}

In this experiment, we use a different data distribution to further evaluate the teacher models. We will examine LR and SVM by classifying uniform spherical data.

\par
\textbf{Teaching in the same feature space.} From Fig. \ref{exp2}, one can observe that the convergence is consistently improved while using omniscient teacher to provide examples to learners. We find that the significance of improvement is related to the training data distribution and loss function, as indicated by our theoretical results. The surrogate teacher produces less convergence gain in SVM, because the convexity lower bound becomes very loose in this case. Overall, omniscient teacher still presents strong teaching capability. More interestingly, we use simple SGD run on the sample set selected by the omniscient teacher and also get faster convergence, showing that the selected example set is better than the entire set in terms of convergence.
\par
\textbf{Teaching in different feature spaces.} While the teacher and student use different feature spaces, one can observe from Fig. \ref{exp2} that the surrogate teacher performs very poorly, even worse than the original SGD and BGD. The imitation teacher works much better and achieves consistent and significant convergence speedup, showing its superiority while the teacher and the student use different features.

\begin{figure*}[h]
  \centering
  \footnotesize
  \includegraphics[width=6.8in]{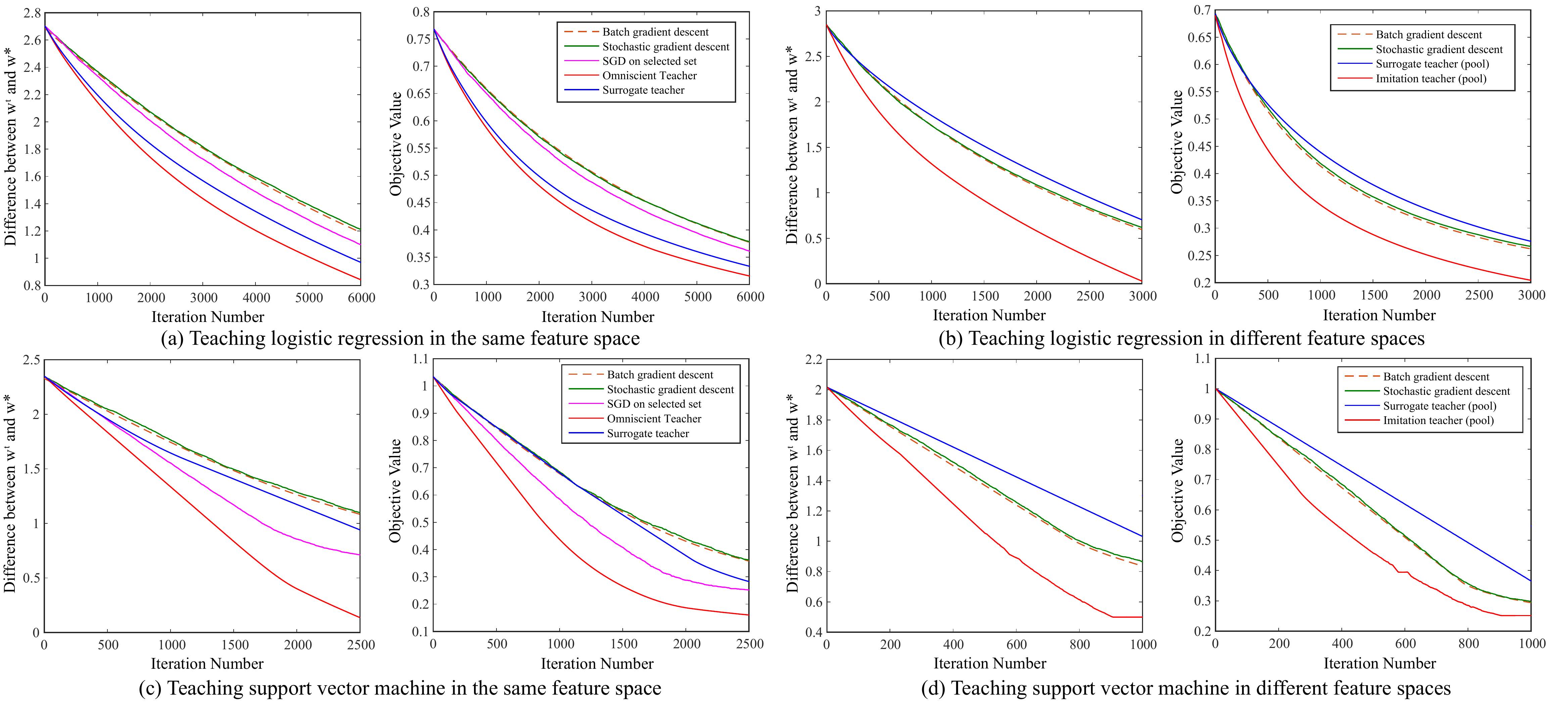}
  \caption{Convergence results on uniform spherical data.}\label{exp2}
\end{figure*}

\clearpage
\newpage

\section{Object learning experiment on children's ego-centric visual data}
\label{appendix:infant}

We experiment with a dataset capturing children and parents interacting with toys in a naturalistic setting~\cite{yurovsky2013}. These interactions are recorded for around 10.5 minutes with a camera worn low on the child's forehead. The head-camera's visual field was 90 degrees wide, providing a broad view of objects visible to the infant. The camera was attached to a headband that was tightened so that it did not move once set on the child. To calibrate the camera, the experimenter noted when the child focused on an object and adjusted the camera until the object was in the center of the image in the control monitor.

For our experiments, we selected interactions of 4 one year old infants. For each parent-child dyad, we annotated the bounding box location and category of the toy attended to by the infant at each frame. There are 10 objects in total: doll (34 frames), toy (53 frames), duck (335 frames), frog (2108 frames), helicopter (169 frames), horse (42 frames), mickey (472 frames), phone (394 frames), sheep (119 frames) and tiger (266 frames). We use a VGG-16 network that is pre-trained on Imagenet dataset as our feature extraction. We first extract the 4096D features from these images and then use PCA to reduce the dimension to 64D. Finally, we run our omniscient teacher on these ego-centric data.

One can observe from Fig. \ref{infant_converge} that our omniscient teacher achieves faster convergence than the random teacher. Moreover, we give part of the selected training examples of random teacher and omniscient teacher in Fig. \ref{infant_vis_rand} and Fig. \ref{infant_vis_omtea}, respectively. We visualize the selected samples every 50 iterations from the first iteration to the 10000th iteration. Interestingly, we find that the training samples that are selected by the omniscient teacher consist of contiguous bouts of experience with the same object instance, unlike the random teacher. The adjacent samples are similar and the object changes in a smooth way. These inputs are qualitatively similar in their ordering to the actual visual experiences of infants in our study, as illustrated in Fig.~\ref{infant_vis_natural}. This can be seen as partial algorithmic confirmation of the desirable structural properties of children's natural learning environment, which emphasizes a smooth and continuous evolution of visual experience, in sharp contrast to random sample selection.

\begin{figure}[h]
  \centering
  \footnotesize
  \includegraphics[width=6.8in]{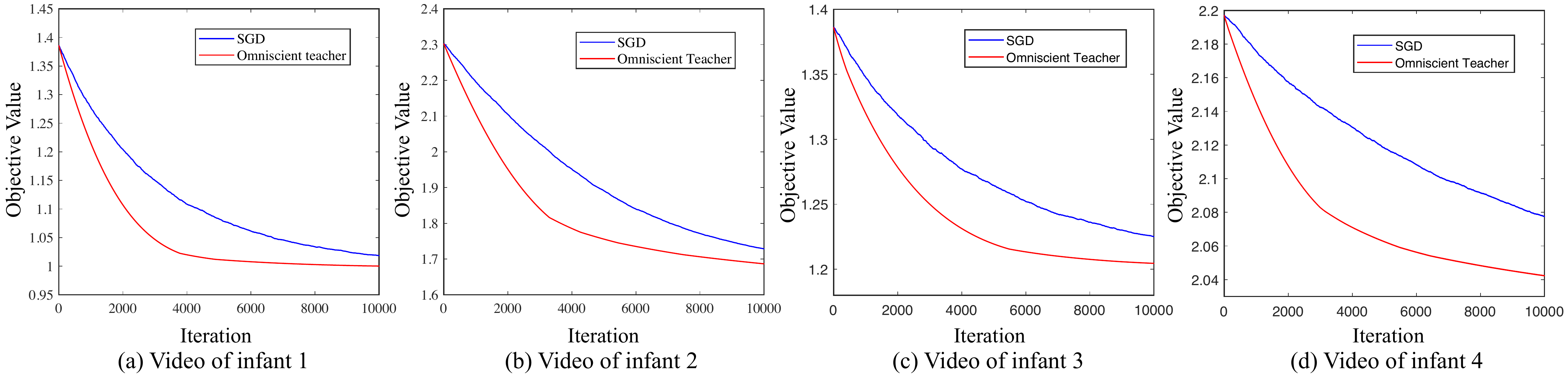}
  \vspace{-2mm}
  \caption{Convergence comparison on infant ego-centric visual data.}\label{infant_converge}
\end{figure}

\begin{figure}[h]
  \centering
  \footnotesize
  \includegraphics[width=4.5in]{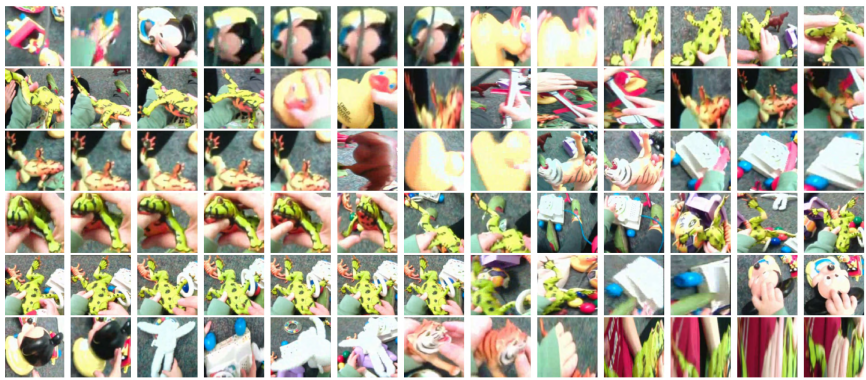}
  \vspace{-2mm}
  \caption{Training examples corresponding to the natural sequence of objects experienced by a single infant in our study.}\label{infant_vis_natural}
\end{figure}

\begin{figure}[h]
  \centering
  \footnotesize
  \includegraphics[width=6.8in]{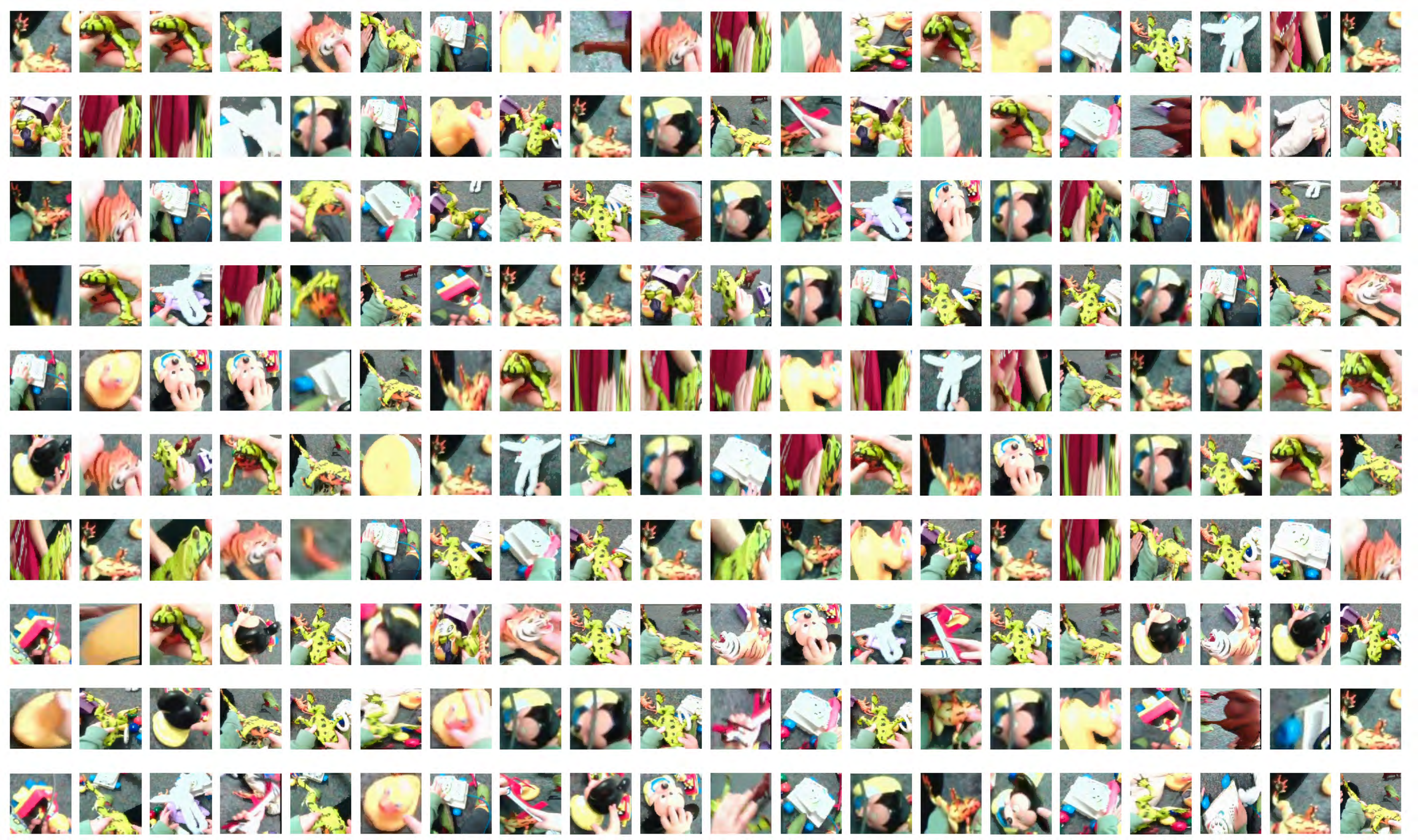}
  \vspace{-2mm}
  \caption{Training examples selected by the random teacher (Stochastic gradient descent).}\label{infant_vis_rand}
\end{figure}
\begin{figure}[h]
  \centering
  \footnotesize
  \includegraphics[width=6.8in]{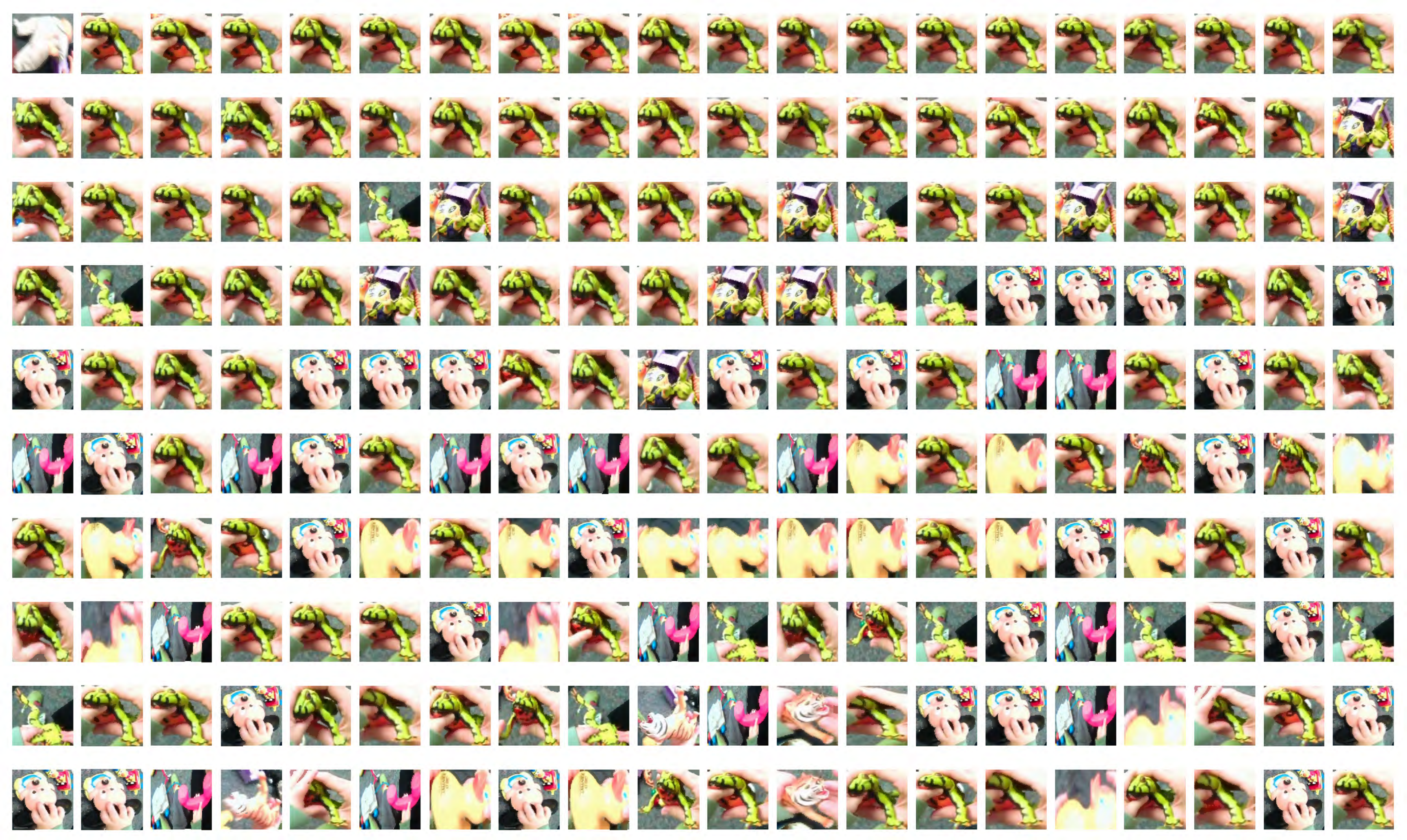}
  \vspace{-2mm}
  \caption{Training examples selected by the omniscient teacher.}\label{infant_vis_omtea}
\end{figure}

\end{document}